\documentclass{article}

\usepackage{microtype}
\usepackage{graphicx}
\usepackage{subcaption}
\usepackage{booktabs} % for professional tables
\usepackage{ascmac}
\usepackage{hyperref}

\usepackage[accepted]{icml2024}

\usepackage{amsmath}
\usepackage{amssymb}
\usepackage{mathtools}
\usepackage{amsthm}
\usepackage{bbm}
\usepackage{cancel}
\usepackage[capitalize,noabbrev]{cleveref}

\theoremstyle{plain}
\newtheorem{theorem}{Theorem}[section]

\newtheorem{lemma}[theorem]{Lemma}

\theoremstyle{definition}

\newtheorem{assumption}[theorem]{Assumption}
\theoremstyle{remark}

\DeclareMathOperator*{\argmin}{arg\,min}
\DeclareMathOperator*{\argmax}{arg\,max}
\newcommand{\E}{{\mathbb E}}
\newcommand{\underE}[2]{{\E}_{\begin{subarray}{c}#1 \end{subarray}}\left[ #2 \right]}

\usepackage[textsize=tiny]{todonotes}

\icmltitlerunning{RVI-SAC}

\begin{document}

\twocolumn[
  \icmltitle{RVI-SAC:\\
    Average Reward Off-Policy Deep Reinforcement Learning}

  % It is OKAY to include author information, even for blind
  % submissions: the style file will automatically remove it for you
  % unless you've provided the [accepted] option to the icml2024
  % package.

  % List of affiliations: The first argument should be a (short)
  % identifier you will use later to specify author affiliations
  % Academic affiliations should list Department, University, City, Region, Country
  % Industry affiliations should list Company, City, Region, Country

  % You can specify symbols, otherwise they are numbered in order.
  % Ideally, you should not use this facility. Affiliations will be numbered
  % in order of appearance and this is the preferred way.
  \icmlsetsymbol{equal}{*}

  \begin{icmlauthorlist}
    \icmlauthor{Yukinari Hisaki}{titech}
    \icmlauthor{Isao Ono}{titech}
  \end{icmlauthorlist}

  \icmlaffiliation{titech}{Tokyo Institute of Technology Yokohama, Kanagawa, Japan}

  \icmlcorrespondingauthor{Yukinari Hisaki}{hiskai.y@ic.c.titech.ac.jp}
  \icmlcorrespondingauthor{Isao Ono}{isao@c.titech.ac.jp}

  % You may provide any keywords that you
  % find helpful for describing your paper; these are used to populate
  % the "keywords" metadata in the PDF but will not be shown in the document
  \icmlkeywords{Machine Learning, ICML}

  \vskip 0.3in
]

% this must go after the closing bracket ] following \twocolumn[ ...

% This command actually creates the footnote in the first column
% listing the affiliations and the copyright notice.
% The command takes one argument, which is text to display at the start of the footnote.
% The \icmlEqualContribution command is standard text for equal contribution.
% Remove it (just {}) if you do not need this facility.

\printAffiliationsAndNotice{}  % leave blank if no need to mention equal contribution
% \printAffiliationsAndNotice{\icmlEqualContribution} % otherwise use the standard text.

\begin{abstract}
  In this paper, we propose an off-policy deep reinforcement learning (DRL) method utilizing the average reward criterion.
  While most existing DRL methods employ the discounted reward criterion, this can potentially lead to a discrepancy between the training objective and performance metrics in continuing tasks, making the average reward criterion a recommended alternative.
  We introduce RVI-SAC, an extension of the state-of-the-art off-policy DRL method, Soft Actor-Critic (SAC) \cite{Haarnoja2018SAC,Haarnoja2018SACv2}, to the average reward criterion.
  Our proposal consists of (1) Critic updates based on RVI Q-learning \cite{Abounadi2001RVI}, (2) Actor updates introduced by the average reward soft policy improvement theorem, and (3) automatic adjustment of Reset Cost enabling the average reward reinforcement learning to be applied to tasks with termination.
  We apply our method to the Gymnasium's Mujoco tasks, a subset of locomotion tasks, and demonstrate that RVI-SAC shows competitive performance compared to existing methods.
\end{abstract}

\section{Introduction}

Model-free reinforcement learning aims to acquire optimal policies through interaction with the environment.
Particularly, Deep Reinforcement Learning (DRL), which approximates functions such as policy or value using Neural Networks, has seen rapid development in recent years.
This advancement has enabled solving tasks with high-dimensional continuous action spaces, such as those in OpenAI Gym's Mujoco tasks \cite{todorov2012mujoco}.
In the realm of DRL methods applicable to tasks with high-dimensional continuous action spaces, methods such as TRPO \cite{Schulman2015TRPO}, PPO \cite{Schulman2017PPO}, DDPG \cite{Silver2014DPG, Lillicrap2016DDPG}, TD3 \cite{Fujimoto2018TD3}, and SAC \cite{Haarnoja2018SAC,Haarnoja2018SACv2} are well-known.
These methods utilize the discounted reward criterion, which is applicable to a variety of MDP-formulated tasks \cite{Puterman1994MDP}.
In particular, for continuing tasks where there is no natural breakpoint in episodes, such as in robot locomotion \cite{todorov2012mujoco} or Access Control Queuing Tasks\cite{Sutton2018RL}, where the interaction between an agent and an environment can continue indefinitely, the discount rate plays a role in keeping the infinite horizon return bounded.
However, discounting introduces an undesirable effect in continuing tasks by prioritizing rewards closer to the current time over those in the future.
An approach to mitigate this effect is to bring the discount rate closer to 1, but it is commonly known that a large discount rate can lead to instability and slower convergence\cite{Fujimoto2018TD3,Dewanto2021ExaminingAR}.

In recent years, the average reward criterion has begun to gain attention as an alternative to the discounted reward criterion.
Reinforcement learning using the average reward criterion aims to maximize the time average of the infinite horizon return in continuing tasks.
In continuing tasks, the average reward criterion is considered more natural than the discounted reward criterion.
Even within the realm of DRL with continuous action space, although few, there have been proposals for methods that utilize the average reward criterion.
Notably, ATRPO \cite{Yiming2021ATRPO} and APO \cite{Xiaoteng2021APO}, which are extensions of the state-of-the-art on-policy DRL methods, TRPO and PPO, to the average reward criterion, have been reported to demonstrate performance on par with or superior to methods using the discounted reward criterion in Mujoco tasks \cite{todorov2012mujoco}.

Research combining the average reward criterion with off-policy methods, generally known for their higher sample efficiency than on-policy approaches, remains limited.
There are several theoretical and tabular approaches to off-policy methods using the average reward criterion, including R-learning \cite{Schwartz1993RLearning,Singh1999RLearning}, RVI Q-learning \cite{Abounadi2001RVI}, Differential Q-learning \cite{Wan2020DifferentialQLearning}, and CSV Q-learning \cite{Yang2016CSVQLearning}.
However, to our knowledge, the only off-policy DRL method with continuous action space that employs the average reward criterion is ARO-DDPG \cite{Saxena2023ARODDPG} which optimize determinisitic policy.
In DRL research using discount rate, Maximum Entropy Reinforcement Learning, which optimizes stochastic policies for entropy-augmented objectives, is known to improve sample efficiency significantly.
Off-policy DRL methods with continuous action space that have adopted this concept \cite{Mnih2016A3C,Haarnoja2017SoftQ,Haarnoja2018SAC,Haarnoja2018SACv2,Abbas2018MPO} have achieved success.
However, to the best of our knowledge, there are no existing DRL methods using the average reward criterion as powerful and sample-efficient as Soft-Actor Critic.

Our goal is to propose RVI-SAC, an off-policy Actor-Critic DRL method that employs the concept of Maximum Entropy Reinforcement Learning under the average reward criterion.
In our proposed method, we use a new Q-Network update method based on RVI Q-learning to update the Critic.
Unlike Differential Q-learning, which was the basis for ARO-DDPG, RVI Q-learning can uniquely determine the convergence point of the Q function \cite{Abounadi2001RVI,Wan2020DifferentialQLearning}.
We identify problems that arise when naively extending RVI Q-learning to a Neural Network update method and address these problems by introducing a technique called Delayed f(Q) Update, enabling the extension of RVI Q-learning to Neural Networks.
We also provide an asymptotic convergence analysis of RVI Q-learning with the Delayed f(Q) Update technique, in a tabular setting using ODE.
Regarding the update of the Actor, we construct a new policy update method that guarantees the improvement of soft average reward by deriving an average reward soft policy improvement theorem, based on soft policy improvement theorem in discounted reward \cite{Haarnoja2018SAC,Haarnoja2018SACv2}.

Our proposed approach addresses the key challenge in applying the average reward criterion to tasks that are not purely continuing, such as in robotic locomotion tasks, for example, Mujoco's Ant, Hopper, Walker2d and Humanoid.
In these tasks, robots may fall, leading to the termination of an episode, which is not permissible in average reward reinforcement learning that aims to optimize the time average of the infinite horizon return.
Similar to ATRPO \cite{Yiming2021ATRPO}, our method introduces a procedure in these tasks that, after a fall, gives a penalty reward (Reset Cost) and resets the environment.
In ATRPO, the Reset Cost was provided as a hyperparameter.
However, the optimal Reset Cost is non-trivial, and setting a sub-optimal Reset Cost could lead to decreased performance.
In our proposed method, we introduce a technique for automatically adjusting the Reset Cost by formulating a constrained optimization problem where the frequency of environment resets, which is independent of the reward scale, is constrained.

Our main contributions in this work are as follows:
\begin{itemize}
  \item We introduce a new off-policy Actor-Critic DRL method, RVI-SAC, utilizing the average reward criterion.
        This method is comprised of two key components:
        (1) a novel Q-Network update approach, RVI Q-learning with the Delayed f(Q) Update technique, and
        (2) a policy update method derived from the average reward soft policy improvement theorem.
        We further provide an asymptotic convergence analysis of RVI Q-learning with the Delayed f(Q) Update technique in a tabular setting using ODE.
  \item To adapt our proposed method for tasks that are not purely continuing, we incorporate environment reset and Reset Cost\cite{Yiming2021ATRPO}.
        By formulating a constrained optimization problem with a constraint based on the frequency of environment resets, an independent measure of the reward scale, we propose a method for automatically adjusting the Reset Cost.
  \item Through benchmark experiments using Mujoco, we demonstrate that our proposed method exhibits competitive performance compared to SAC\cite{Haarnoja2018SACv2} with various discount rates and ARO-DDPG \cite{Saxena2023ARODDPG}.
\end{itemize}

\section{Preliminaries}

In this section, we introduce problem setting and average reward reinforcement learning, which is the core concept of our proposed method.
The mathematical notations employed throughout this paper are detailed in Appendix \ref{sec:mathematical_notation}.

\subsection{Markov Decision Process}

We define the interaction between the environment and the agent as a Markov Decision Process (MDP) $\mathcal{M} = (\mathcal{S}, \mathcal{A}, r, p)$.
Here, $s \in \mathcal{S}$ represents the state space, $a \in \mathcal{A}$ represents the action space, $r: \mathcal{S} \times \mathcal{A} \rightarrow \mathbb{R}, |r(\cdot, \cdot)| \leq \|r\|_{\infty} $ is the reward function, and $p: \mathcal{S} \times \mathcal{A} \times \mathcal{S} \to [0, 1]$ is the state transition function.
At each discrete time step $t=0,1,2,\ldots$, the agent receives a state $S_t \in \mathcal{S}$ from the MDP and selects an action $A_t \in \mathcal{A}$.
The environment then provides a reward $R_t = r(S_t, A_t)$ and the next state $S_{t+1} \in \mathcal{S},$ repeating this process.
The state transition function is defined for all $s, s^{\prime} \in \mathcal{S}, a \in \mathcal{A}$ as $p\left(s^{\prime}\mid s, a\right) := \operatorname{Pr}\left(S_{t+1}=s^{\prime} \mid S_t=\right.$ $\left.s, A_t=a\right)$.
Furthermore, we use a stationary Markov policy $\pi: \mathcal{S} \times \mathcal{A} \rightarrow [0, 1]$ as the criterion for action selection.
This represents the probability of selecting an action $a \in \mathcal{A}$ given a state $s \in \mathcal{S}$ and is defined as $\pi(a|s) := \operatorname{Pr}\left(A_t=a \mid S_t=s\right)$.

\subsection{Average Reward Reinforcement Learning}
\label{sec:average_reward_reinforcement_learning}

To simplify the discussion that follows, we make the following assumption for the MDPs where average reward reinforcement learning is applied:
\begin{assumption}
  \label{assumption:ergodic}
  For any policy $\pi$, the MDP $\mathcal{M}$ is ergodic.
\end{assumption}
Under this assumption, for any policy $\pi$, a stationary state distribution $d_{\pi}(s)$ exists.
The distribution including actions is denoted as $d_\pi(s, a) = d_\pi(s) \pi(s|a)$.

The average reward for a policy $\pi$ is defined as:
\begin{equation}
  \label{eq:average_reward}
  \rho^{\pi} := \lim_{T \rightarrow \infty} \frac{1}{T} \underE{\pi}{\sum_{t=0}^{T} R_t} = \sum_{s \in \mathcal{S}, a \in \mathcal{A}} d_\pi(s,a) r(s,a) .
\end{equation}
The optimal policy in average reward criterion is defined as:
\begin{equation}
  \pi^* = \argmax_{\pi} \rho^{\pi}.
\end{equation}

The Q function for average reward reinforcement learning is defined as:
\begin{equation}\label{eq:average_reward_Q}
  Q^\pi(s,a) := \underE{\pi}{\sum_{t=0}^{\infty} (R_t - \rho^\pi) | S_0 = s, A_0 = a} .
\end{equation}

The optimal Bellman equation for average reward is as follows:
\begin{equation}
  \label{eq:average_reward_optimal_bellman_equation}
  \begin{aligned}
    Q(s, a) = r(s, a) - \rho + \sum_{s' \in \mathcal{S}} p(s'|s, a) \max_{a'} Q(s', a'), \\
    \forall (s,a) \in \mathcal{S} \times \mathcal{A} .
  \end{aligned}
\end{equation}
From existing research (e.g., \citet{Puterman1994MDP}), it is known that this equation has the following properties:
\begin{itemize}
  \item There exists a unique solution for $\rho$ as $\rho = \rho^{\pi^*}$.
        % \item For $Q(s,a)$, the solution is determined as $q(s,a) = Q^{\pi^*}(s,a) + c$ for any constant $c \in \mathbb{R}$.
  \item There exisits a unique solution only up to a constant $c \in \mathbb{R}$ for $Q(s,a)$ as $q(s,a) = Q^{\pi^*}(s,a) + c$.
  \item For the solution $Q$ to the equation, a deterministic policy $\mu(s) = \argmax_{a} q(s,a)$ is one of the optimal policies.
\end{itemize}
Henceforth, we denote $\rho^{\pi^*}$ as $\rho^*$.

\subsection{RVI Q-learning}
\label{sec:RVI_Q_learning}

RVI Q-learning \cite{Konda1999ActorCriticSA,Wan2020DifferentialQLearning} is one of the average reward reinforcement learning methods for the tabular Q function, and updates the Q function as follows:
\begin{eqnarray}
  \label{eq:RVI_Q_learning}
  \begin{aligned}
     & Q_{t+1}(S_t, A_t)  = Q_t(S_t, A_t) +                                              \\
     & \alpha_t \left( R_t - f(Q_t) + \max_{a'} Q_t (S_{t+1}, a') - Q_t(S_t, A_t)\right).
  \end{aligned}
\end{eqnarray}
From the convergence proof of generalized RVI Q-learning \cite{Wan2020DifferentialQLearning}, the function $f$ can be any function that satisfies the following assumption:
\begin{assumption}[From \citet{Wan2020DifferentialQLearning}]
  $f: \mathbb{R}^{|\mathcal{S} \times \mathcal{A}|} \rightarrow \mathbb{R}$ is Lipschitz, and there exists some $u>0$ such that for all $c \in \mathbb{R}$ and $x \in \mathbb{R}^{|\mathcal{S} \times \mathcal{A}|}, f(e)=u$, $f(x+c e)=f(x)+c u$ and $f(c x)=c f(x)$.
\end{assumption}
In practice, $f$ often takes forms such as $f(Q) = Q(S,A), \max_{a} Q(S,a)$ using arbitrary state/action pair $(S,A)$ as the Reference State, or the sum over all state/action pairs, $f(Q) = g \sum_{(s,a) \in \mathcal{S}\times\mathcal{A}} Q(s,a), g \sum_{s \in \mathcal{S}} \max_{a} Q(s,a)$, with gain $\forall g > 0$.
This assumption plays an important role in demonstrating the convergence of the algorithm. Intuitively, this indicates that the function $f$, satisfying this assumption, can include functions that correspond to specific elements of the vector $x$ or their weighted linear sum.

This algorithm, under certain appropriate assumptions, converges \textit{almost surely} to a unique solution $q^*$, and it has been shown that $q^*$ satisfies both the optimal Bellman equation (Equation \ref{eq:average_reward_optimal_bellman_equation}) and
$$
  \rho^* = f(q^*)
$$
as demonstrated in \citet{Wan2020DifferentialQLearning}.

\section{Proposed Method}

We propose a new off-policy Actor-Critic DRL method based on average reward.
To this end, in Section \ref{sec:RVI_Q_learning_based_Q_Network_update},
we present a Critic update method based on RVI Q-learning,
and in Section \ref{sec:Average_Reward_Soft_Policy_Improvement_Theorem},
we demonstrate an Actor update method based on SAC \cite{Haarnoja2018SAC,Haarnoja2018SACv2}.
Additionally, in Section \ref{sec:Automatic_Reset_Cost_Adjustment}, we introduce a method to apply average reward reinforcement learning to problems that are not purely continuing tasks, such as locomotion tasks with termination.
The overall algorithm is detailed in Appendix \ref{sec:Overall_RVI_SAC_algorithm_and_implementation}.

\subsection{RVI Q-learning based Q-Network update}
\label{sec:RVI_Q_learning_based_Q_Network_update}

We extend the RVI Q-learning algorithm described in Equation \ref{eq:RVI_Q_learning} to an updating method for the Q function represented by a Neural Network.
Following traditional approach in Neural Network-based Q-learning, we update the parameters $\phi$ of the Q-Network $Q_{\phi}$ using the target:
$$
  Y(r, s') = r - f(Q_{\phi'}) + \max_{a'} Q_{\phi'}(s', a'),
$$
and by minimizing the following loss function:
$$
  \frac{1}{|\mathcal{B}|} \sum_{(s, a, r, s') \in \mathcal{B}}\left(Y(r, s') - Q_{\phi}(s, a)\right)^2 .
$$
$\mathcal{B}$ is a mini-batch uniformly sampled from the Replay Buffer $\mathcal{D}$, which accumulates experiences $(S_t, A_t, R_t, S_{t+1})$ obtained during training, and $\phi'$ are the parameters of the target network \cite{Mnih2013DQN}.

In implementing this method, we need to consider the following two points:

The first point is the choice of function $f$.
As mentioned in Section \ref{sec:RVI_Q_learning}, tabular RVI Q-learning typically uses a Reference State $(S, A) \in \mathcal{S} \times \mathcal{A}$ or the sum over all state/action pairs to calculate $f(Q)$.
Using the Reference State is easily applicable to problems with continuous state/action spaces in Neural Network-based methods.
However, concerns arise about performance dependency on the visitation frequency to the Reference State and the accuracy of its Q-value \cite{Wan2020DifferentialQLearning}.
On the other hand, calculating the sum over all state/action pairs does not require a Reference State but is not directly computable with Neural Networks for $f(Q_{\phi'})$.
To address these issues, we substitute $f(Q_{\phi'})$ with $f(Q_{\phi'}; \mathcal{B})$, calculated using mini-batch $\mathcal{B}$, as shown in Equation \ref{eq:F_Sampled}:
\begin{equation}
  \label{eq:F_Sampled}
  f(Q_{\phi'}; \mathcal{B}) = \frac{1}{|\mathcal{B}|} \sum_{s \in \mathcal{B}} \max_{a} Q_{\phi'}(s, a).
\end{equation}
Equation \ref{eq:F_Sampled} serves as an unbiased estimator of $f(Q_{\phi'})$ when set as:
\begin{equation}\label{eq:F_Sampled_unbiased_estimator}
  f(Q_{\phi'}) = \underE{s \sim d_{b}(\cdot)}{\max_{a} Q_{\phi'}(s, a)},
\end{equation}
where $b$ represents the behavior policy in off-policy methods, and $d_{b}(\cdot)$ denotes the stationary state distribution under the behavior policy $b$.
In our method, we use settings as shown in Equations \ref{eq:F_Sampled} and \ref{eq:F_Sampled_unbiased_estimator}, but this discussion is applicable to any setting that satisfies:
\begin{equation}
  \label{eq:F_Sampled_general}
  f(Q_{\phi'}) = \underE{X_t}{f(Q_{\phi'}; X_t)}
\end{equation}
for the random variable $X_t$. Thus, the target value used for the Q-Network update becomes:
\begin{equation}
  \label{eq:Q_Network_target1}
  Y(r, s'; \mathcal{B}) = r - f(Q_{\phi'}; \mathcal{B}) + \max_{a'} Q_{\phi'}(s', a').
\end{equation}

The second point is that the variance of the sample value $f(Q_{\phi'}; \mathcal{B})$ (Equation \ref{eq:F_Sampled}) can increase the variance of the target value (Equation \ref{eq:Q_Network_target1}), potentially leading to instability in learning.
This issue is pronounced when the variance of the Q-values is large.
A high variance in Q-values can potentially lead to an increase in the variance of the target values, creating a feedback loop that might further amplify the variance of Q-values.
To mitigate the variance of the target value, we propose the \textbf{Delayed f(Q) Update} technique.
Delayed f(Q) Update employs a value $\xi_t$, updated as follows, instead of using $f(Q_{\phi'}; \mathcal{B})$ for calculating the target value:
$$
  \xi_{t+1} = \xi_{t} + \beta_t \left(f(Q_{\phi'};\mathcal{B}) - \xi_t\right),
$$
$\beta_t$ denotes the learning rate for $\xi_t$. The new target value using $\xi_t$ is then:
$$
  Y(r, s'; \xi_t) = r - \xi_t + \max_{a'} Q_{\phi'}(s', a').
$$
In this case, $\xi_t$ serves as a smoothed value of $f(Q_{\phi'}; \mathcal{B})$, and this update method is expected to reduce the variance of the target value.

\textbf{Theoretical Analysis of Delayed f(Q) Update}

We reinterpret the Q-Network update method using Delayed f(Q) Update in the context of a tabular Q-learning algorithm, it can be expressed as follows:
\small
\begin{equation}
  \begin{aligned}
    \label{eq:RVI_Q_learning_with_Delayed_fQ_Update}
    Q_{t+1}(S_t, A_t) & = Q_{t}(S_t, A_t) +                                                                   \\
                      & \alpha_{t} \left(R_t - \xi_t + \max_{a'} Q_{t}(S_{t+1}, a') - Q_{t}(S_t, A_t)\right), \\
    \xi_{t+1}         & = \xi_{t} + \beta_t \left(f(Q_t; X_t) - \xi_t\right).
  \end{aligned}
\end{equation}
\normalsize
This update formula is a specialization of asynchronous stochastic approximation (SA) on two time scales \cite{Borkar1997TwoTimeScalesSA,Konda1999ActorCriticSA}.
By selecting appropriate learning rates such that $\frac{\alpha_t}{\beta_t} \rightarrow 0$, the update of $\xi_t$ in Equation \ref{eq:RVI_Q_learning_with_Delayed_fQ_Update} can be considered faster relative to the update of $Q_t$.
Therefore, since $\xi_t$ can be considered a ``relative constant", it can be viewed as being equivalent to $f(Q_t)$.
The convergence of this algorithm is discussed in Appendix \ref{sec:convergence_proof_of_RVI_Q_learning_with_Delayed_fQ_Update} and summarized in Theorem \ref{thm:convergence_of_RVI_Q_learning_with_Delayed_fQ_Update}:
\begin{theorem}[Sketch]
  \label{thm:convergence_of_RVI_Q_learning_with_Delayed_fQ_Update}
  The algorithm expressed by the following equations converges \textit{almost surely} to a uniquely determined $q^*$ under appropriate assumptions (see Appendix \ref{sec:convergence_proof_of_RVI_Q_learning_with_Delayed_fQ_Update}):
  \small
  $$
    \begin{aligned}
      Q_{t+1}(S_t, A_t) & = Q_{t}(S_t, A_t) +                                                                                                              \\
                        & \alpha_{t} \left(R_t - \hat{\xi}_t + \max_{a'} Q_{t}(S_{t+1}, a') - Q_{t}(S_t, A_t)\right),                                      \\
      \xi_{t+1}         & = \xi_{t} + \beta_t \left(f(Q_t; X) - \xi_t\right),                                                                              \\
                        & \textup{where } \hat{\xi}_t =  \textup{{clip}}(\xi_t, - \|r\|_\infty - \epsilon, \|r\|_\infty + \epsilon), \forall \epsilon > 0.
    \end{aligned}
  $$
  \normalsize
\end{theorem}

\subsection{Average Reward Soft Policy Improvement Theorem}
\label{sec:Average_Reward_Soft_Policy_Improvement_Theorem}

We propose a policy update method for the average reward criterion, inspired by the policy update method of the SAC\cite{Haarnoja2018SAC,Haarnoja2018SACv2}.

In the SAC, a soft-Q function defined for the discount rate $\gamma \in (0,1)$ and a policy $\pi$ \cite{Haarnoja2017SoftQ}:
\begin{equation}\label{eq:soft_Q_function}
  \begin{aligned}
     & Q^{\pi, \gamma}(s, a)  :=                                                                                         \\
     & \mathbb{E}_{\pi}{\left[\sum_{t=0}^{\infty} \gamma^t (R_t - \log \pi(A_{t+1}|S_{t+1})) | S_0 = s, A_0 = a\right]}.
  \end{aligned}
\end{equation}

The policy is then updated in the following manner for $\forall s \in \mathcal{S}$:
\begin{equation}\label{eq:soft_policy_update}
  \pi_{\text{new}}(\cdot|s) = \argmin_{\pi \in \Pi} D_{\text{KL}} \left( \pi(\cdot|s) \middle| \frac{\exp{(Q^{\pi_{\text{old}}, \gamma}(s, \cdot))}}{Z^{\pi_{\text{old}}}(s)} \right). \\
\end{equation}
The partition function $Z^{\pi_{\text{old}}}(s)$ normalizes the distribution and can be ignored in the gradient of the new policy (refer to \citet{Haarnoja2018SAC}).
$\Pi$ represents the set of policies, such as a parametrized family like Gaussian policies.
According to the soft policy improvement theorem \cite{Haarnoja2018SAC,Haarnoja2018SACv2}, for the updated policy $\pi_{\text{new}}$, the condition $Q^{\pi_{\text{new}}, \gamma}(s, a) \geq Q^{\pi_{\text{old}}, \gamma}(s, a), \forall (s, a) \in \mathcal{S} \times \mathcal{A}$ holds, indicating policy improvement.
The actor's update rule in the SAC is constructed based on this theorem.
Further, defining the entropy-augmented reward $R_t^{\text{ent}} := R_t - \mathbb{E}_{s' \sim p(\cdot|S_t, A_t), a' \sim \pi(\cdot|s')}{\left[\log \pi(a'|s')\right]}$, the Q function in Equation \ref{eq:soft_Q_function} can be reformulated as $Q^{\pi, \gamma}(s, a) := \mathbb{E}_{\pi}{\left[\sum_{t=0}^{\infty} \gamma^t R_t^{\text{ent}} | S_0 = s, A_0 = a\right]}$, allowing the application of the standard discounted Q-learning framework for the critic's update \cite{Haarnoja2018SAC,Haarnoja2018SACv2}.

In the context of average reward reinforcement learning, the soft average reward is defined as:
\begin{equation}
  \label{eq:soft_average_reward}
  \rho^{\pi}_{\text{soft}} = \lim_{T \rightarrow \infty} \frac{1}{T} \mathbb{E}_{\pi}{\left[\sum_{t=0}^{T} R_t - \log \pi(A_{t}|S_{t})\right]}.
\end{equation}
Correspondingly, the average reward soft-Q function is defined as:
\begin{equation}
  \label{eq:soft_average_reward_Q}
  \begin{aligned}
     & Q^{\pi}(s, a) :=                                                                                                                  \\
     & \mathbb{E}_{\pi}{\left[\sum_{t=0}^{\infty} R_t - \rho^{\pi}_{\text{soft}} - \log \pi(A_{t+1}|S_{t+1}) | S_0 = s, A_0 = a\right]}.
  \end{aligned}
\end{equation}

From Equation \ref{eq:soft_average_reward_Q}, the soft Q function represents the expected cumulative sum of rewards minus the average reward.
Thus, the relationship $Q^{\pi_{\text{new}}}(s, a) \geq Q^{\pi_{\text{old}}}(s, a), \forall (s, a) \in \mathcal{S} \times \mathcal{A}$ in the policy improvement theorem for the discounted SAC does not guarantee policy improvement.
We present a new average reward soft policy improvement theorem using the soft average reward $\rho^\pi_{\text{soft}}$ as a metric.
\begin{theorem}[Average Reward Soft Policy Improvement]
  \label{thm:Average_Reward_Soft_Policy_Improvement}
  Let $\pi_{\text{old}} \in \Pi$ and let $\pi_{\text{new}}$ be the optimizer of the minimization problem defined in Equation \ref{eq:soft_policy_update}.
  Then $\rho^{\pi_{\text{new}}} \geq \rho^{\pi_{\text{old}}}$ holds.
\end{theorem}
\begin{proof}
  See Appendix \ref{sec:Soft_Policy_Improvement}.
\end{proof}
This result demonstrates that updating the policy in the same manner as SAC leads to improvements in the policy under the average reward criterion.
Additionally, defining the entropy-augmented reward $R_t^{\text{ent}}$ and the entropy-augmented Q function $Q^{\pi, \text{ent}}(s, a)$ as
\begin{eqnarray}
  R_t^{\text{ent}} &:=& R_t - \log \pi(A_t|S_t),\nonumber\\
  Q^{\pi, \text{ent}}(s, a) &:=& Q^{\pi}(s, a) - \log \pi(A_t|S_t). \label{eq:soft_average_reward_Q_ent}
\end{eqnarray}
allows the Q function in Equation \ref{eq:soft_average_reward_Q} to be reformulated as $Q^{\pi, \text{ent}}(s, a) := \mathbb{E}_{\pi}{\left[\sum_{t=0}^{\infty} R_t^{\text{ent}} - \rho^{\pi}_{\text{soft}} | S_0 = s, A_0 = a\right]}$.
This formulation aligns with the definition of the Q function in average reward reinforcement learning (Equation \ref{eq:average_reward_Q}), enabling the application of the average reward Q-learning framework.

\subsection{Automatic Reset Cost Adjustment}
\label{sec:Automatic_Reset_Cost_Adjustment}

In this section, we address the challenge associated with applying the average reward criterion to tasks that are not purely continuing tasks, such as locomotion tasks where episodes may end due to falls.
Average reward reinforcement learning assumes continuing tasks that do not have an episode termination.
This is because average rewards are defined over the infinite horizon, and after the end of an episode, the agent continues to receive a reward of zero, leading to an average reward of zero.
However, in many tasks, such as locomotion tasks, episodes may end due to events like robot falls, depending on the policy.
In these cases, the tasks are not purely continuing.

% To apply average reward reinforcement learning to such tasks, we employ the environment reset and the Reset Cost as in ATRPO .
To apply average reward reinforcement learning to such tasks, we employ the environment reset and the Reset Cost which ATRPO \cite{Yiming2021ATRPO} does.
The environment reset regards a terminated episode as a continuing one by initializing the environment.
Reset Cost is the penalty reward given for resetting the environment, denoted as $-r_{\text{cost}}$ (where $r_{\text{cost}} > 0$).
This means that, even after an episode ends in a certain terminal state $S_t$, initializing the environment, and observing the initial state $S_0$, the experience $(S_{t-1}, A_{t-1}, r(S_{t-1}, A_{t-1}) - r_{\text{cost}}, S_0)$ is obtained, and the episode is treated as continued.

In ATRPO, for experiments in the Mujoco environment, the Reset Cost $r_{\text{cost}}$ is fixed at 100, but the optimal Reset Cost is generally non-trivial.
Instead of setting the $r_{\text{cost}}$, we propose a method to control the frequency at which the agent reaches termination states.
Let's consider a new MDP from MDPs with termination, where we only introduce environment resets without adding the Reset Cost (equivalent to the environment when $r_{\text{cost}}=0$).
Let $\mathcal{S}_{\text{term}}$ be the set of termination states in the original MDP, and define the frequency $\rho_{\text{reset}}^\pi$ at which the agent reaches termination states under the policy $\pi$ as follows:
$$
  \rho_{\text{reset}}^\pi = \lim_{T \rightarrow \infty} \frac{1}{T} \mathbb{E}_{\pi}{ \left[ \sum_{t=0}^{T} \underE{s' \sim p(\cdot | S_t, A_t )}{ \mathbbm{1}\left(s' \in \mathcal{S}_{\text{term}}\right) }  \right] }.
$$
Using $\rho_{\text{reset}}^\pi$, we consider the following constrained optimization problem:
\begin{equation}
  \label{eq:reset_cost_optimization}
  \begin{aligned}
     & \max_{\pi} \rho^\pi,                                               \\
     & \text{s.t. } \rho_{\text{reset}}^\pi \leq \epsilon_{\text{reset}}.
  \end{aligned}
\end{equation}
This problem aims to maximize the average reward $\rho^\pi$ with a constraint on the frequency of reaching termination states, where the termination frequency target $\epsilon_{\text{reset}} \in (0, 1)$ is a user parameter.
Note that $\rho^\pi$ here refers to the average reward when the Reset Cost is set to zero.

To solve this constrained optimization problem, we define the Lagrangian for the dual variable $\lambda$ as follows:
$$
  \mathcal{L}(\pi, \lambda) = \rho^\pi - \lambda \rho_{\text{reset}}^\pi - \lambda \epsilon_{\text{reset}}.
$$
Following prior research in constrained optimization problems, the primal problem is formulated as:
$$
  \max_{\pi} \min_{\lambda \geq 0} \mathcal{L}(\pi, \lambda).
$$
In our approach to solving this problem, similar to the adjustment of the temperature parameter in Maximum Entropy Reinforcement Learning \cite{Haarnoja2018SACv2,Abbas2018MPO}, we alternate between outer and inner optimization steps.
The outer optimization step is updating $\pi$ to maximize $\rho^\pi - \lambda\rho_{\text{reset}}^\pi$ for a fixed $\lambda$.
Since $\rho^\pi - \lambda\rho_{\text{reset}}^\pi$ is equal to the average reward when $r_{\text{cost}}=\lambda$, this optimization step is equivalent to the policy update step in average reward reinforcement learning with Reset Cost.
The inner optimization step is updating $\lambda$ to minimize $- \lambda \rho_{\text{reset}}^\pi - \lambda \epsilon_{\text{reset}}$.
To compute this objective, it is necessary to obtain $\rho_{\text{reset}}^\pi$.
Hence, we estimate the value of $\rho_{\text{reset}}^\pi$ by updating the Q function $Q_{\text{reset}}$ under the setting $r(s,a) = \underE{s' \sim p(\cdot | s, a )}{ \mathbbm{1}\left(s' \in \mathcal{S}_{\text{term}}\right) } $ using the update method described in Section \ref{sec:RVI_Q_learning_based_Q_Network_update}.

\section{Experiment}

In our benchmark experiments, we aim to verify two aspects:
(1) A comparison of the performance between RVI-SAC, SAC\cite{Haarnoja2018SACv2} with various discount rates, and the existing off-policy average reward DRL method, ARO-DDPG \cite{Saxena2023ARODDPG}.
(2) How does each component in our proposed method contribute to performance?

To demonstrate these, we conducted benchmark experiments using six tasks (Ant, HalfCheetah, Hopper, Walker2d, Humanoid, and Swimmer) implemented in the Gymnasium \cite{towers_gymnasium_2023} and MuJoCo physical simulator \cite{todorov2012mujoco}.
Note that there is no termination in the Swimmer and HalfCheetah environments, meaning that resets do not occur.

The source code for this experiment can be found on our GitHub repository at \url{https://github.com/yhisaki/average-reward-drl}.

\subsection{Comparative evaluation}
\label{sec:comparative_evaluation}

\begin{figure*}[t]
    \vspace{-0.2cm}
    \begin{subfigure}{1.0\textwidth}
        \centering
        \includegraphics[scale=0.5]{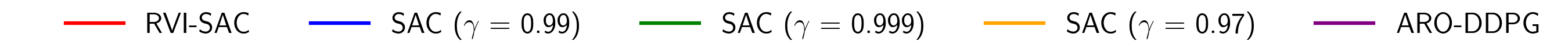}
    \end{subfigure}
    \centering
    \begin{subfigure}{.3\textwidth}
        \centering
        \includegraphics[width=\linewidth]{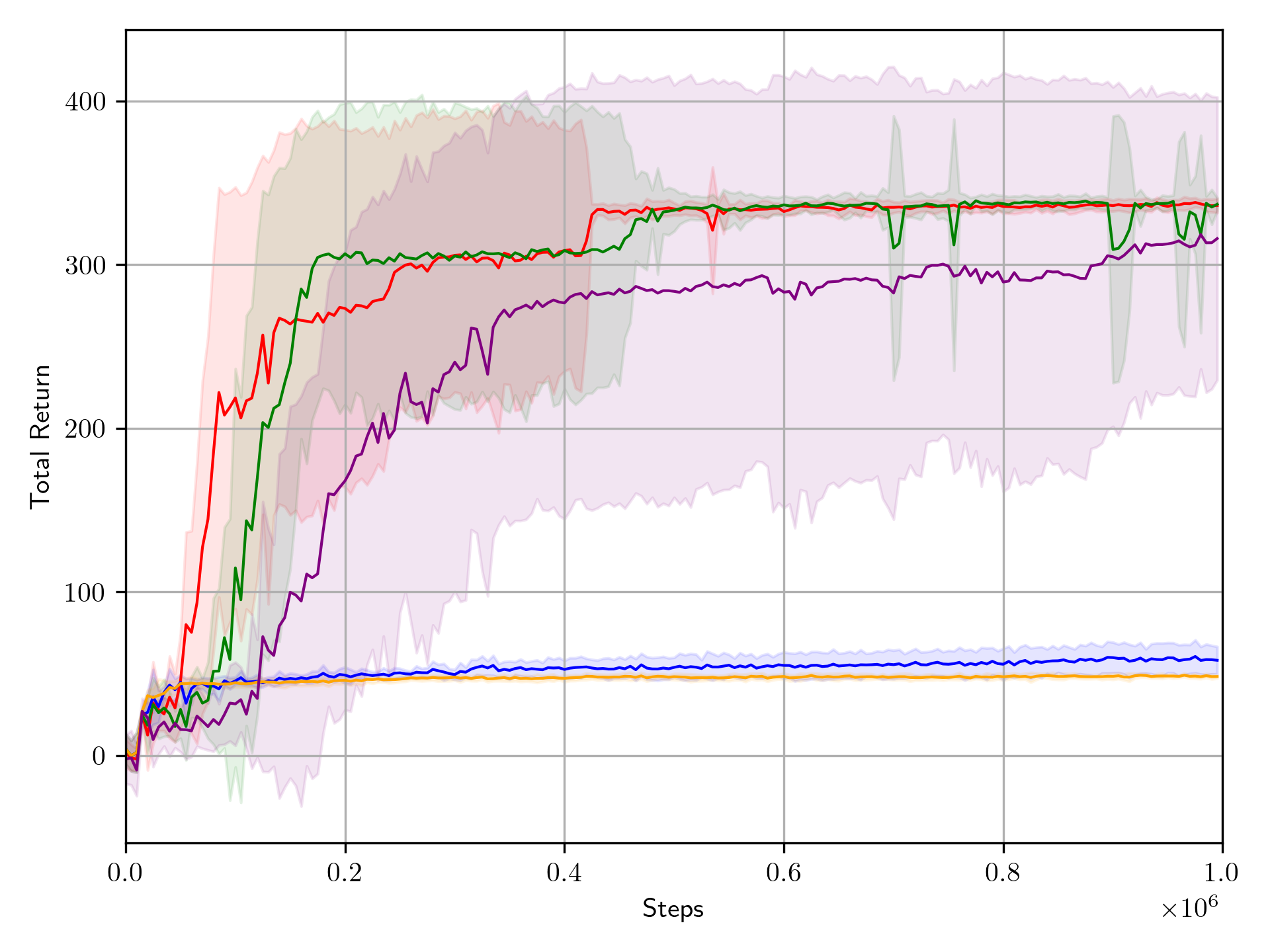}
        \caption{Swimmer}
        \label{fig:main_result_swimmer}
    \end{subfigure}
    \begin{subfigure}{.3\textwidth}
        \centering
        \includegraphics[width=\linewidth]{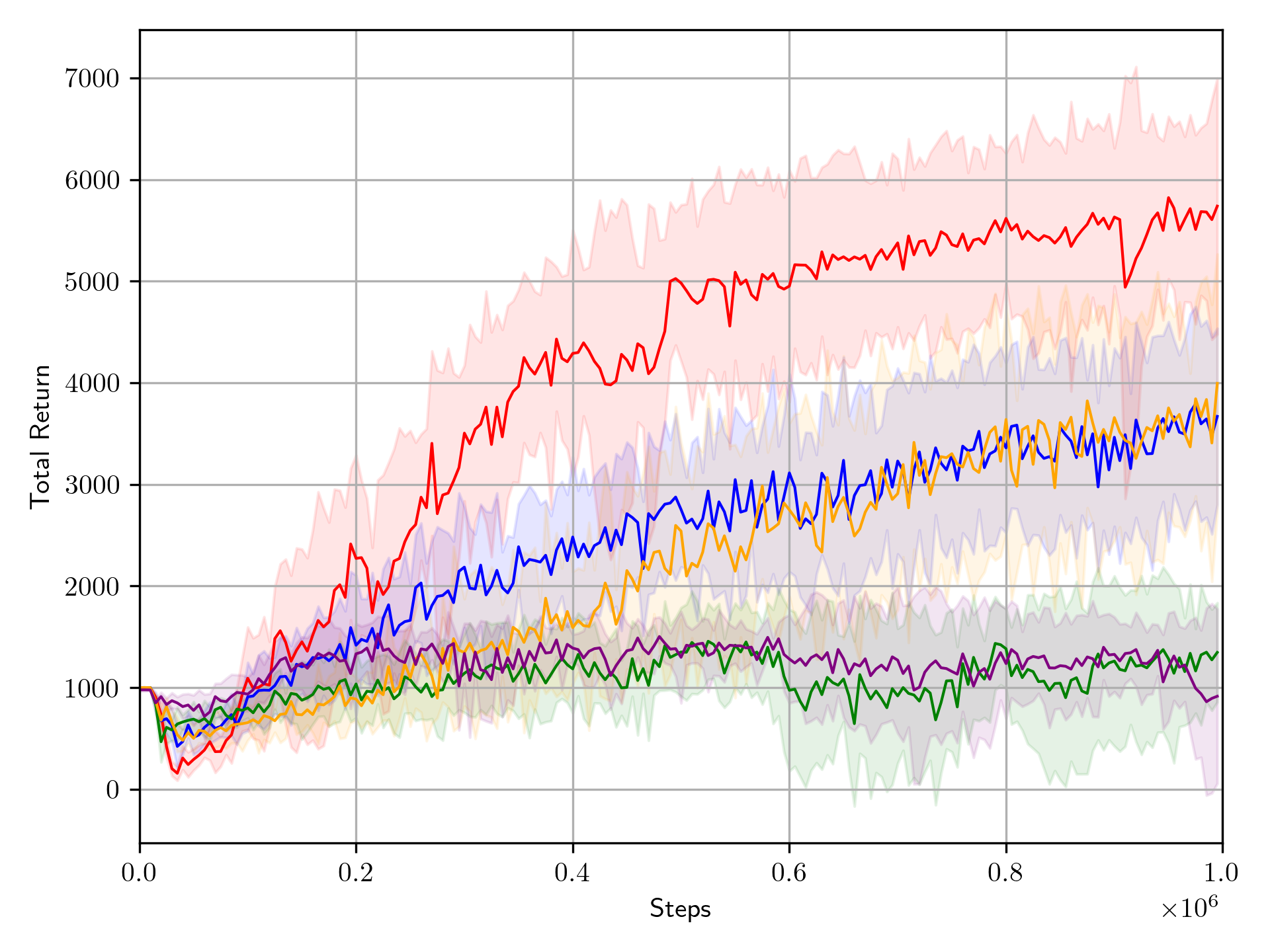}
        \caption{Ant}
        \label{fig:main_result_ant}
    \end{subfigure}
    \begin{subfigure}{.3\textwidth}
        \centering
        \includegraphics[width=\linewidth]{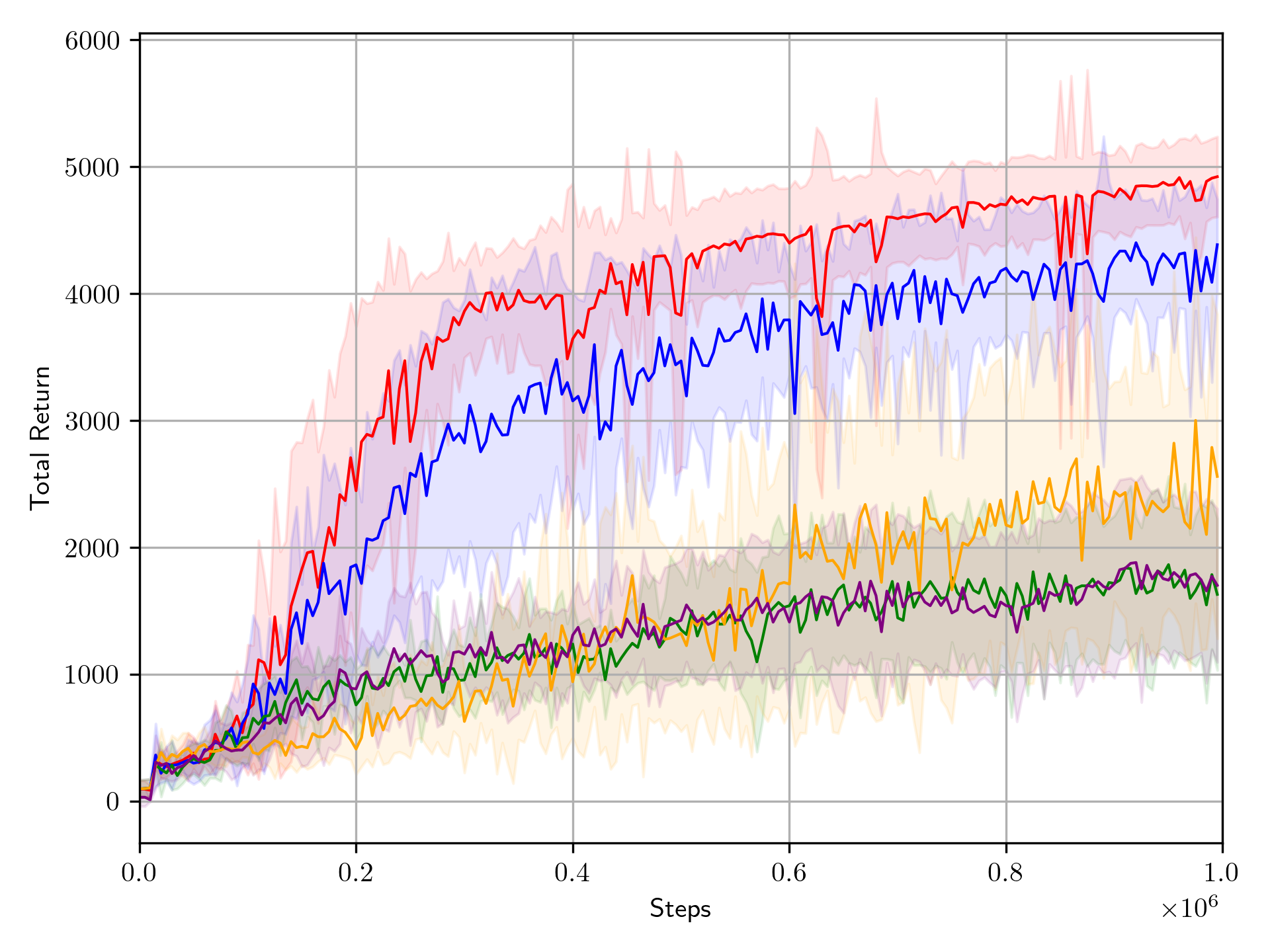}
        \caption{Walker2d}
        \label{fig:main_result_walker2d}
    \end{subfigure}%

    \begin{subfigure}{.3\textwidth}
        \centering
        \includegraphics[width=\linewidth]{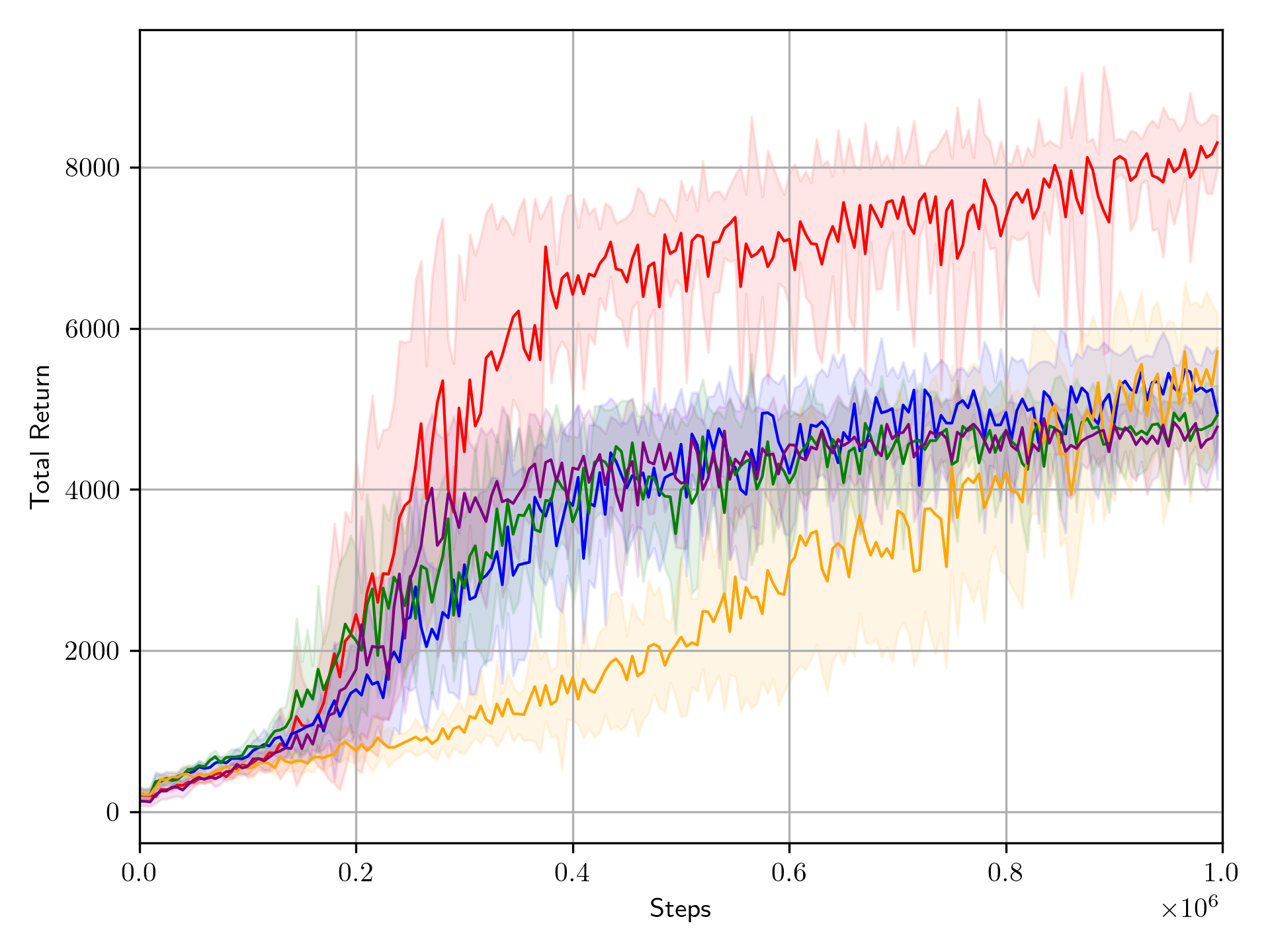}
        \caption{Humanoid}
        \label{fig:main_result_humanoid}
    \end{subfigure}
    \begin{subfigure}{.3\textwidth}
        \centering
        \includegraphics[width=\linewidth]{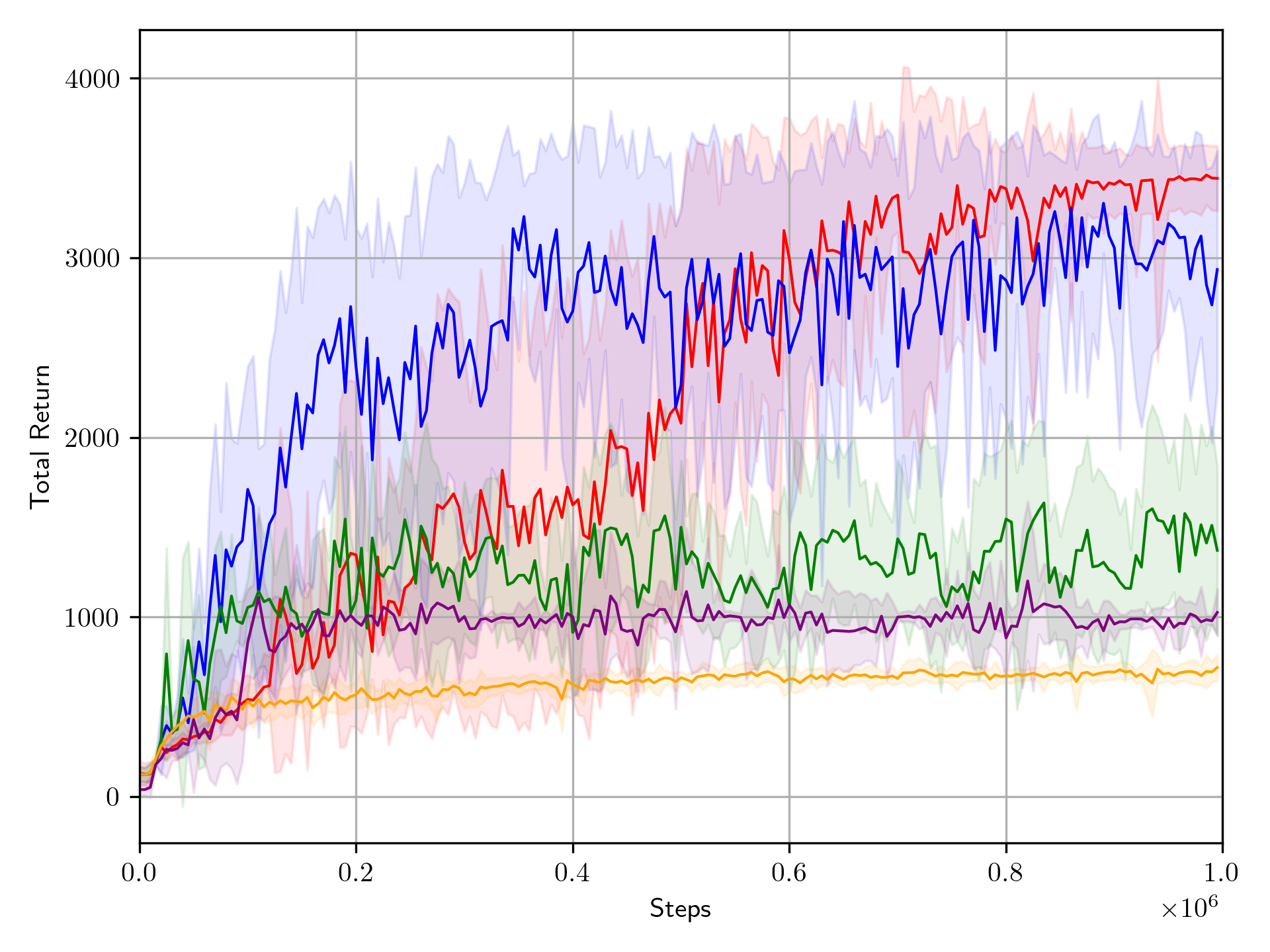}
        \caption{Hopper}
        \label{fig:main_result_hopper}
    \end{subfigure}
    \begin{subfigure}{.3\textwidth}
        \centering
        \includegraphics[width=\linewidth]{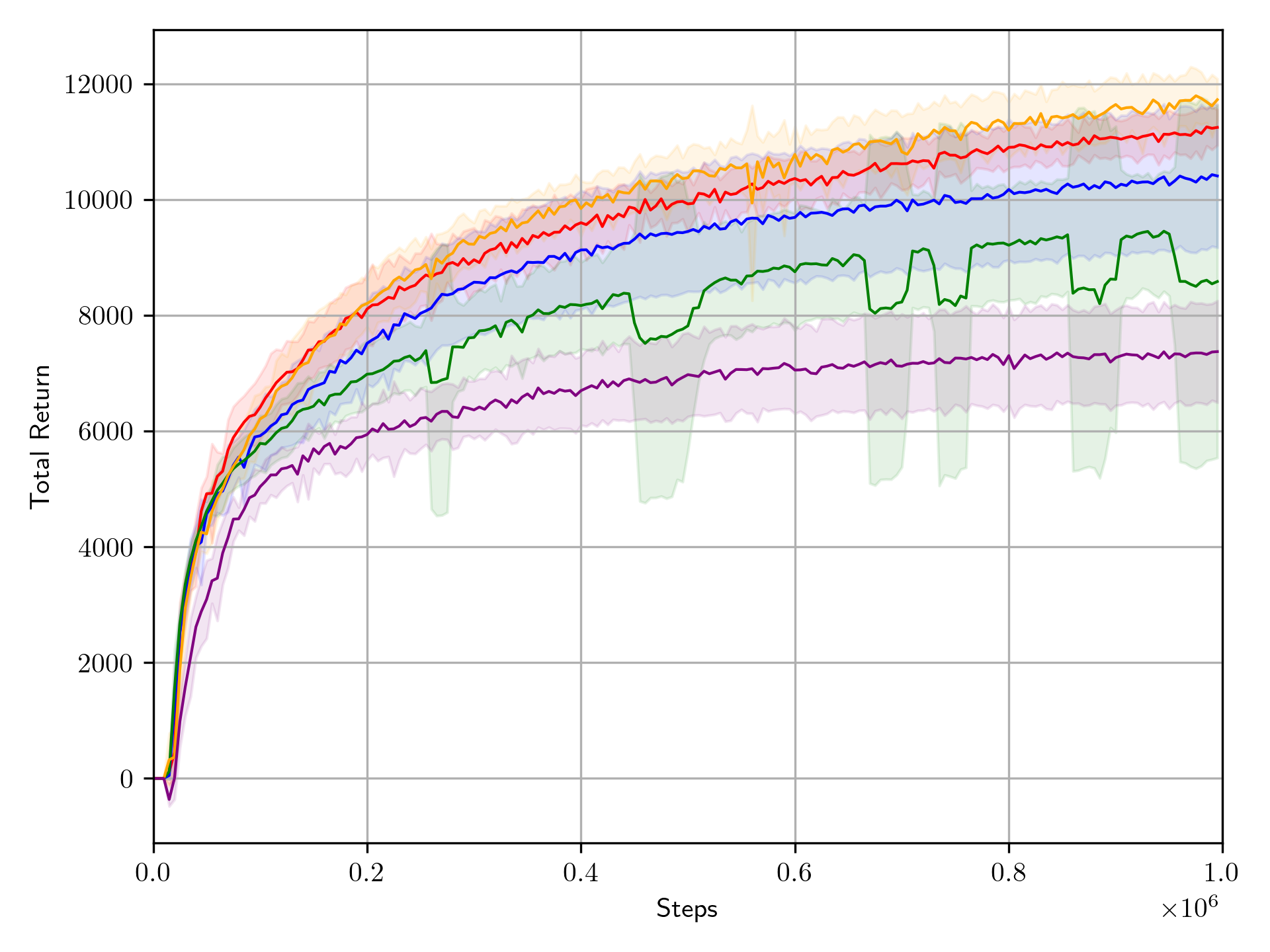}
        \caption{HalfCheetah}
        \label{fig:main_result_halfcheetah}
    \end{subfigure}
    \vspace{-0.2cm}
    \caption{Learning curves for the Gymnasium's Mujoco tasks. The horizontal axis represents Steps, and the vertical axis represents the evaluation value (total\_return).
        Lines and shades represent the mean and standard deviation of the evaluation values over 10 trials, respectively.}
    \label{fig:main_result}
    \vspace{-0.5cm}
\end{figure*}

We conducted experiments with 10 different random seed trials for each algorithm, sampling evaluation scores every 5,000 steps.
For RVI-SAC and SAC, stochastic policies are treated as deterministic during evaluation.
The maximum episode step is limited to 1,000 during training and evaluation.
Figure \ref{fig:main_result} shows the learning curves of RVI-SAC, SAC with various discount rates, and ARO-DDPG.
These experiments set the evaluation score as the total return over 1,000 steps.
Results of experiments with the evaluation score set as an average reward (total\_return / survival\_step) are presented in Appendix \ref{sec:appendix_average_reward_evaluation}.

From the results shown in Figure \ref{fig:main_result}, when comparing RVI-SAC with SAC with various discount rates, RVI-SAC demonstrates equal or better performance than SAC with the best discount rate in all environments except HalfCheetah.
A notable observation is from the Swimmer environment experiments (Figure \ref{fig:main_result_swimmer}).
SAC's recommended discount rate of $\gamma = 0.99$ \cite{Haarnoja2018SAC,Haarnoja2018SACv2} performs better than the other rates in environments other than Swimmer.
However, a larger discount rate of $\gamma = 0.999$ is required in the Swimmer environment.
However, setting a large discount rate can lead to destabilization of learning and slow convergence \cite{Fujimoto2018TD3,Dewanto2021ExaminingAR}, and indeed, in the environments other than Swimmer, a setting of $\gamma = 0.999$ shows lower performance.
Compared to SAC, RVI-SAC shows the same performance as SAC ($\gamma = 0.999$) in the Swimmer environment and equal or better than SAC ($\gamma = 0.99$) in the other environments.
This result suggests that while traditional SAC using a discount rate may be significantly impacted by the choice of discount rate, RVI-SAC using the average reward resolves this issue.

When comparing RVI-SAC with ARO-DDPG, RVI-SAC shows higher performance in all environments.
SAC has improved performance over methods using deterministic policies by introducing the concept of Maximum Entropy Reinforcement Learning.
Similarly, it can be considered that the introduction of this concept to RVI-SAC is the primary reason for RVI-SAC's superior performance over ARO-DDPG.

\subsection{Design evaluation}
\label{sec:design_evaluation}

\begin{figure*}[t]
    \centering
    \begin{subfigure}{.32\textwidth}
        \centering
        \includegraphics[width=\linewidth]{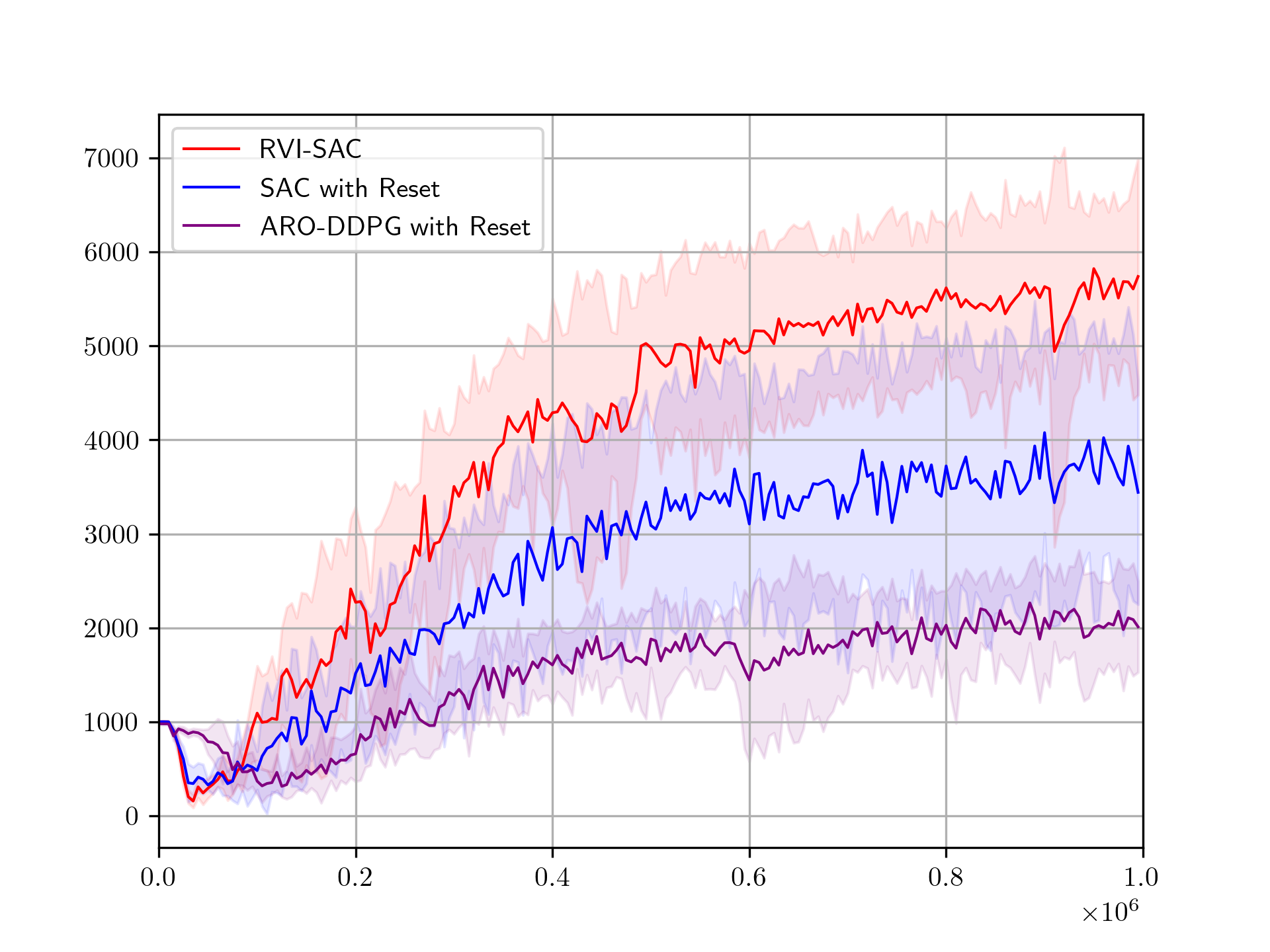}
        % \caption{Discounted SAC with Reset. 赤はRVI-SACを，青はDiscounted SACに環境のResetを導入したものを示している．}
        % \caption{Perfomance Comparison with SAC and ARO-DDPG with Reset}
        % \caption{This figure shows the effectiveness of using the average reward criterion by comparing RVI-SAC (red) with SAC with Reset (blue) and ARO-DDPG with Reset (purple).\\ \\ }
        \caption{Performance Comparison of RVI-SAC, SAC with automatic Reset Cost adjustment, and ARO-DDPG with automatic Reset Cost adjustment}
        % \vspace{6mm}
        \label{fig:sac_with_reset}
    \end{subfigure}
    \begin{subfigure}{.32\textwidth}
        \centering
        \includegraphics[width=\linewidth]{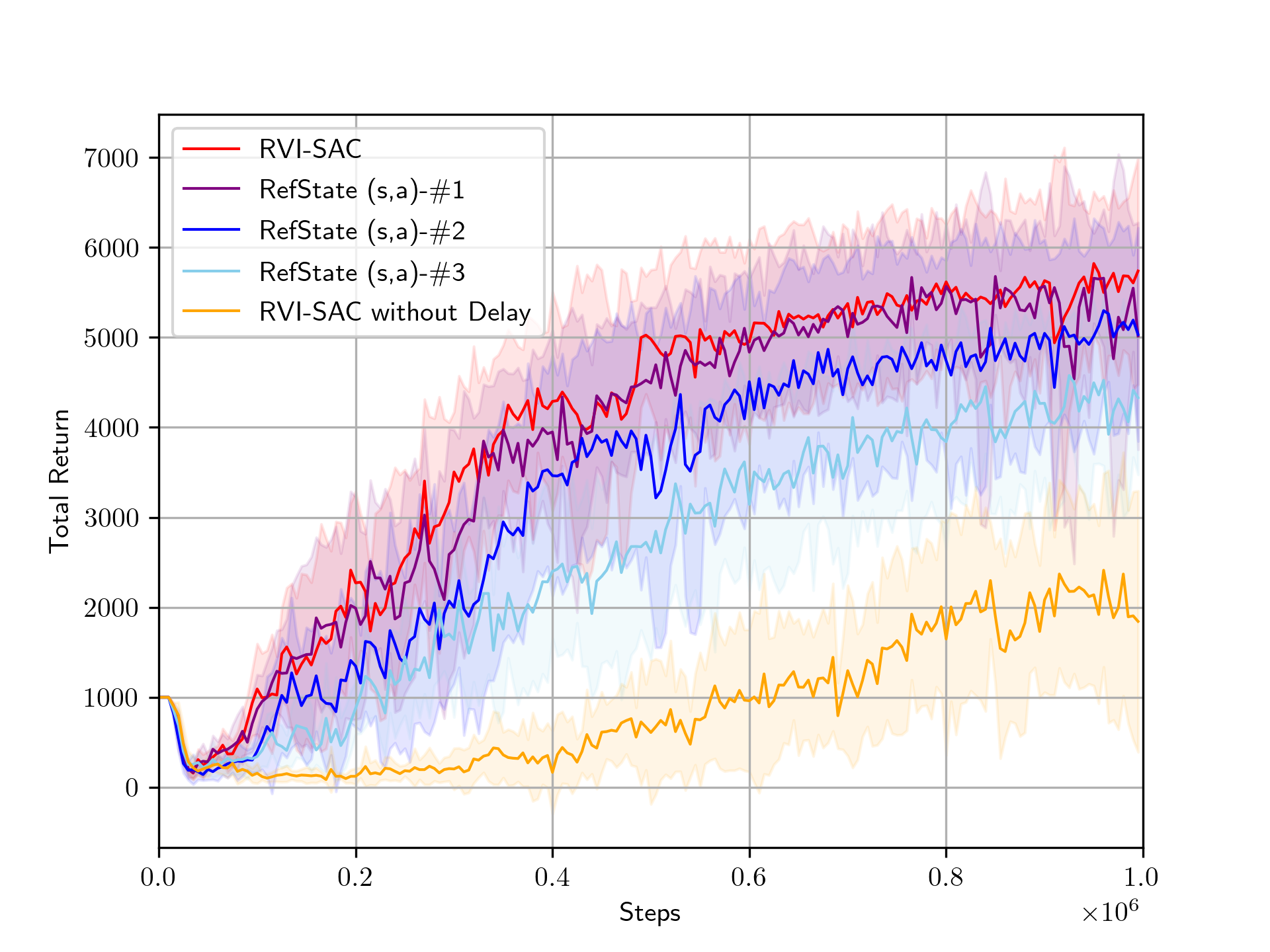}
        \caption{Ablation Study of Delayed f(Q) Update\newline \newline \newline }
        % \caption{This figure shows the effectiveness of the Delayed f(Q) Update by comparing RVI-SAC (red) with RVI-SAC variants using Reference States $(s,a)-\#1, (s,a)-\#2, (s,a)-\#3$ (purple, blue, skyblue), and a direct application of Equation \ref{eq:F_Sampled} instead of using Delayed f(Q) Update(orange).}
        % \vspace{3mm}
        \label{fig:rvi_ablation}
    \end{subfigure}
    \begin{subfigure}{.32\textwidth}
        \centering
        \includegraphics[width=\linewidth]{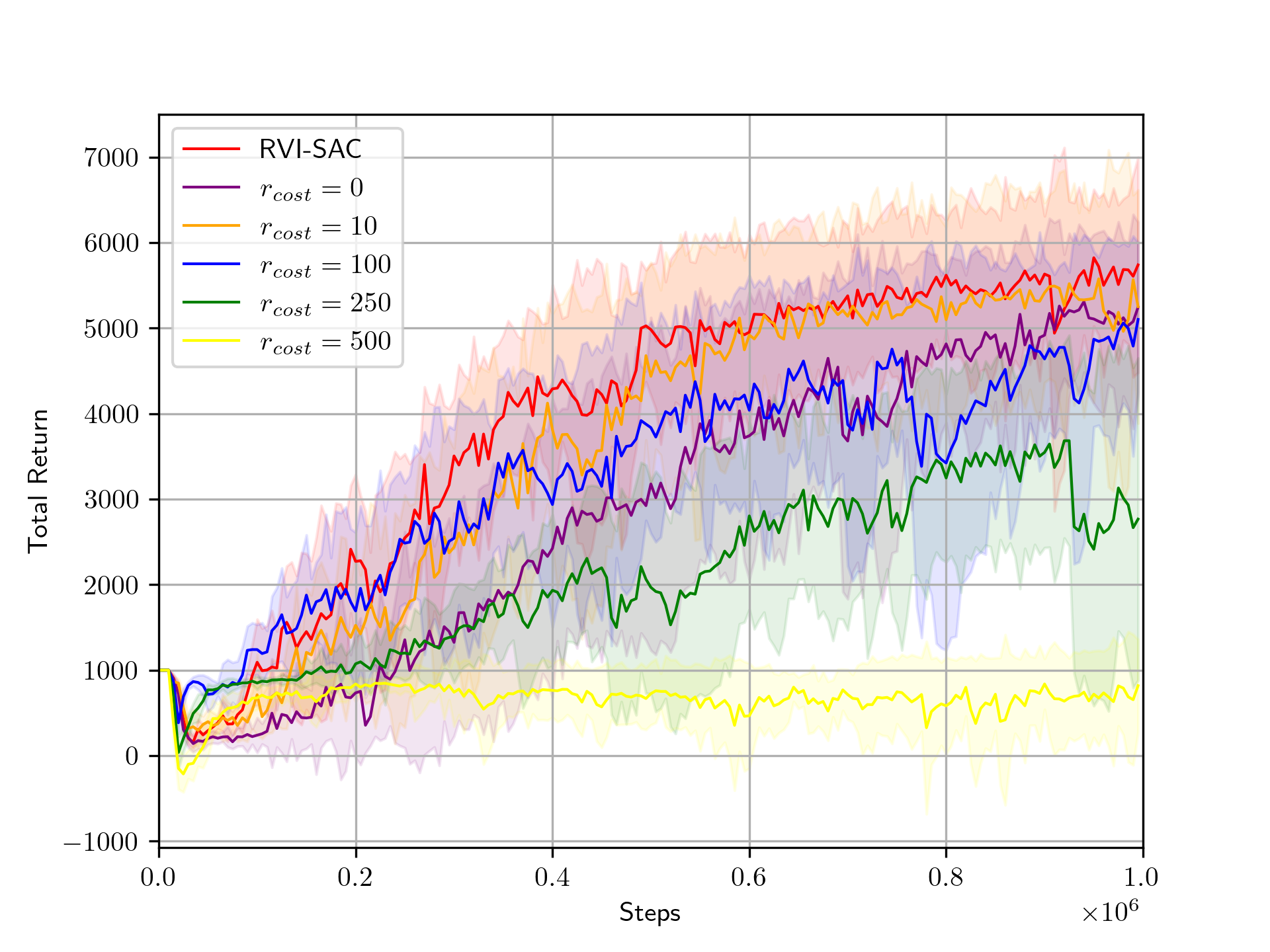}
        \caption{Performance Comparison of RVI-SAC and RVI-SAC with Fixed Reset Cost ($r_{\text{cost}}=0, 10, 100, 250, 500$)\newline }
        % \caption{This figure shows the effectiveness of automatically adjusting the Reset Cost by comparing RVI-SAC (red) with RVI-SAC with fixed Reset Costs set to $r_{cost} = $ 0(purple), 10(orange), 100(blue), 250(green) and 500(yellow).\\ }
        % \vspace{3mm}
        \label{fig:fixed_reset}
    \end{subfigure}
    \vspace{-0.2cm}
    \caption{
        Experimental results demonstrating the effectiveness of each component of RVI-SAC.
        All three graphs represent learning curves on the Ant environment.
        % They compare RVI-SAC (red) with:
        % \textbf{(a)} SAC (blue) and ARO-DDPG(purple) with environment reset and automatic adjustment of Reset Cost,
        % \textbf{(b)} RVI-SAC variants using Reference States $(s,a)-\#1, (s,a)-\#2$ and $(s,a)-\#3$ (purple, blue and skyblue), and RVI-SAC without Delay that is RVI-SAC updating the Q function using Equation \ref{eq:F_Sampled} directory as $f(Q)$.(orange),
        % \textbf{(c)} RVI-SAC with fixed Reset Costs set to $r_{cost} = $ 0(purple), 10(orange), 100(blue), 250(green), 500(yellow).
    }
    \vspace{-0.5cm}
\end{figure*}

In the previous section, we demonstrated that RVI-SAC overall exhibits better performance compared to SAC using various discount rates and ARO-DDPG.
In this section, we show how each component of RVI-SAC contributes to the overall performance.

\textbf{Performance Comparison of RVI-SAC, SAC with automatic Reset Cost adjustment and ARO-DDPG with automatic Reset Cost adjustment}

Since RVI-SAC introduces the automatic Reset Cost adjustment, RVI-SAC uses a different reward structures from that used in SAC and ARO-DDPG in which the reward is set to zero after termination.
To compare the performance of RVI-SAC, SAC and ARO-DDPG under the same reward structure, we conduct comparative experiments with \textbf{RVI-SAC}, SAC with the automatic Reset Cost adjustment (\textbf{SAC with Reset}) and ARO-DDPG with the automatic Reset Cost adjustment(\textbf{ARO-DDPG with Reset}).
Figure \ref{fig:sac_with_reset} shows the learning curves of these experiments in the Ant environment.
(Results for other environments are shown in Appendix \ref{sec:appendix_sac_with_reset}).
Here, the discount rate of SAC is set to $\gamma = 0.99$.

Figure \ref{fig:sac_with_reset} demonstrates that RVI-SAC outperforms SAC with automatic Reset Cost adjustment and ARO-DDPG with automatic Reset Cost adjustment.
This result suggests that the improved performance of RVI-SAC is not solely due to the different reward structure but also due to the effects of using the average reward criterion.

\textbf{Ablation Study of Delayed f(Q) Update}

In this section, we evaluate the effectiveness of the Delayed f(Q) Update described in Section \ref{sec:RVI_Q_learning_based_Q_Network_update}.
This method stabilizes learning without depending on a specific state/action pair for updating the Q function.
To validate the effectiveness of this method, we examine whether the followings are correct:
(1) When the function $f(Q)$ is $f(Q) = Q(S,A)$ for state/action pair $(S,A) \in \mathcal{S} \times \mathcal{A}$, the performance depends on the choice of $(S,A)$.
(2) When the Q function is updated using Equation \ref{eq:F_Sampled} directory as $f(Q)$, learning becomes unstable.
To examine these, we conducted performance comparisons with three algorithms:
(1) \textbf{RVI-SAC},
(2) \textbf{RefState {\boldmath $(s,a)-\#1, (s,a)-\#2$ and $(s,a)-\#3$}} that are RVI-SACs updating the Q functions with $f(Q) = Q(s,a)$, using sampled state/action pairs, $(s,a)-\#1, (s,a)-\#2$ and $(s,a)-\#3$, as Reference States obtained from when agent takes random actions, respectively,
(3) \textbf{RVI-SAC without Delay} that is RVI-SAC updating the Q function using Equation \ref{eq:F_Sampled} directory as $f(Q)$.

Figure \ref{fig:rvi_ablation} shows the learning curves for these methods in the Ant environment.
Firstly, comparing RVI-SAC with RefState $(s,a)-\#1, (s,a)-\#2$ and $(s,a)-\#3$, except for RefState $(s,a)-\#1$, the methods using Reference States show lower performance than RVI-SAC.
Furthermore, comparing the results of RefState $(s,a)-\#1, (s,a)-\#2$ and $(s,a)-\#3$, it suggests that performance depends on the choice of Reference State.
These results suggest the effectiveness of RVI-SAC, which shows good performance without depending on a specific state/action pair.
Next, comparing RVI-SAC with RVI-SAC without Delay, which directly uses Equation \ref{eq:F_Sampled}, it is observed that RVI-SAC performs significantly better.
This result suggests that, in RVI-SAC, implementing the Delayed f(Q) Update contributes to stabilizing the learning process, thereby achieving higher performance.
It indicates the effectiveness of the Delayed f(Q) Update, aiming at stabilizing updates of the Q-function.

\textbf{RVI-SAC with Fixed Reset Cost}

To demonstrate the effectiveness of the Automatic Reset Cost Adjustment, we compare the performance of RVI-SAC and RVI-SAC with fixed Reset Costs in the Ant environment.
Figure \ref{fig:fixed_reset} shows the learning curves of \textbf{RVI-SAC} and \textbf{RVI-SAC with fixed Reset Costs, {\boldmath $r_{\text{cost}}=0, 10, 100, 250, 500$}.}
These results show that settings other than the optimal fixed Reset Cost of $r_{\text{cost}} = 10$ for this environment decrease performance.
Moreover, the performance of RVI-SAC with fixed Reset Costs is highly dependent on its setting.
This result suggests the effectiveness of the automatic adjustment of Reset Cost, which does not require specific settings.

\section{Related Works}

The importance of using the average reward criterion in continuing tasks has been suggested.
\citet{Blackwell1962DiscreteDP} showed that a Blackwell optimal policy $\pi^*$ maximizes the average reward exists, and that, for any discount reward criterion satisfying $\gamma \geq \gamma^*$, the optimal policy coincides with the Blackwell optimal policy.
However, \citet{Dewanto2021ExaminingAR} demonstrated that setting a high discount rate to satisfy $\gamma \geq \gamma^*$ can slow down convergence, and a lower discount rate may lead to a sub-optimal policy.
Additionally, they noted various benefits of directly applying the average reward criterion to recurrent MDPs.
\citet{Naik2019DRLOptimization} pointed out that discounted reinforcement learning with function approximation is not an optimization problem, and the optimal policy is not well-defined.

Although there are fewer studies on tabular Q-learning methods and theoretical analyses for the average reward criterion compared to those for the discounted reward criterion, several notable works exist \cite{Schwartz1993RLearning, Singh1999RLearning, Abounadi2001RVI, Wan2020DifferentialQLearning, Yang2016CSVQLearning} .
RVI Q-learning, which forms the foundational idea of our proposed method, was proposed by \citet{Abounadi2001RVI} and generalized by \citet{Wan2020DifferentialQLearning} with respect to the function $f$.
Differential Q-learning \cite{Wan2020DifferentialQLearning} is a special case of RVI Q-learning.
The asymptotic convergence of these methods in weakly communicating MDPs has been established by \citet{Wan2022OConvergenceWeaklyMDPs}.

Several methods focusing on function approximation for average reward criterion have been proposed \cite{Saxena2023ARODDPG,Chaudhuri2019POLITEX,Yiming2021ATRPO,Xiaoteng2021APO,Zhang2021AROPE,Zhang2021BreakingDT}.
A notable study by \citet{Saxena2023ARODDPG} extended DDPG to the average reward criterion and demonstrated performance using dm-control's Mujoco tasks.
This work is deeply relevant to our proposed method.
This research provides asymptotic convergence and finite time analysis in a linear function approximation.
Another contribution is POLITEX \cite{Chaudhuri2019POLITEX}, which updates a policy using a Boltzmann distribution over the sum of action-value estimates from a prior policy.
POLITEX demonstrated performance using an Atari's Ms Pacman.
ATRPO \cite{Yiming2021ATRPO} and APO \cite{Xiaoteng2021APO} update a policy based on policy improvement bounds to the average reward setting within an on-policy framework.
ATRPO and APO demonstrated performance using OpenAI Gym's Mujoco tasks.

\section{Conclusion}

% In this paper, we proposed RVI-SAC, a novel RL algorithm that can be applied to environments with termination.
% 本論文で我々は，novel off-policy DRL method for the average reward criterionであるRVI-SACを提案した．
% RVI-SACは3つのコンポーネントで構成される．
% １つ目は，RVI Q-learningをベースとしたCriticの更新である．
% 我々はRVI Q-learningをNeural Networkを用いた手法にナイーブに拡張するだけではなく，Delayed f(Q) Updateというテクニックを導入することで，Reference Stateに依存せず安定した学習することをを可能にした．
% 更に，我々はDelayed f(Q) Updateのasymptotic convergenceを証明した．
% 2つ目は，Average Reward Soft Policy Improvement Theoremを用いたActorの更新である．
% 3つ目は，終了状態を持つロコモーションに平均報酬強化学習を適用するためのReset Costの自動調整である．

In this paper, we proposed RVI-SAC, a novel off-policy DRL method with the average reward criterion.
RVI-SAC is composed of three components. The first component is the Critic update based on RVI Q-learning.
We did not simply extend RVI Q-learning to the Neural Network method, but introduced a new technique called Delayed f(Q) Update, enabling stable learning without dependence on a Reference State.
Additionally, we proved the asymptotic convergence of this method.
The second component is the Actor update using the Average Reward Soft Policy Improvement Theorem.
The third component is the automatic adjustment of the Reset Cost to apply average reward reinforcement learning to locomotion tasks with termination.
We applied RVI-SAC to the Gymnasium's Mujoco tasks, and demonstrated that RVI-SAC showed competitive performance compared to existing methods.

% Future workとしては，提案手法のLinear Function Approximationにおけるasympototic convergenceやfinite time analysisを提供すること挙げられる．
For future work, on the theoretical side, we consider to provide asymptotic convergence and finite-time analysis of the proposed method using a linear function approximator.
On the experimental side, we plan to compare the performance of RVI-SAC using benchmark tasks other than Mujoco tasks and compare it with average reward on-policy methods such as APO \cite{Xiaoteng2021APO}.

\clearpage
% In the unusual situation where you want a paper to appear in the
% references without citing it in the main text, use \nocite
\section*{Impact Statement}

This paper presents work whose goal is to advance the field of Reinforcement Learning. There are many potential societal consequences of our work, none which we feel must be specifically highlighted here.

% \nocite{langley00}

\bibliography{references}
\bibliographystyle{icml2024}

%%%%%%%%%%%%%%%%%%%%%%%%%%%%%%%%%%%%%%%%%%%%%%%%%%%%%%%%%%%%%%%%%%%%%%%%%%%%%%%
%%%%%%%%%%%%%%%%%%%%%%%%%%%%%%%%%%%%%%%%%%%%%%%%%%%%%%%%%%%%%%%%%%%%%%%%%%%%%%%
% APPENDIX
%%%%%%%%%%%%%%%%%%%%%%%%%%%%%%%%%%%%%%%%%%%%%%%%%%%%%%%%%%%%%%%%%%%%%%%%%%%%%%%
%%%%%%%%%%%%%%%%%%%%%%%%%%%%%%%%%%%%%%%%%%%%%%%%%%%%%%%%%%%%%%%%%%%%%%%%%%%%%%%
\newpage
\appendix
\onecolumn

\section{Mathematical Notations}
\label{sec:mathematical_notation}

In this paper, we utilize the following mathematical notations:
\begin{itemize}
  \item $e$ denotes a vector with all elements being 1.
  \item $D_{\text{KL}}(p|q)$ represents the Kullback-Leibler divergence, defined between probability distributions $p$ and $q$ as $D_{\text{KL}}(p | q) = \sum_x p(x) \log\frac{p(x)}{q(x)}$.
  \item Here, $\mathbbm{1}(\text{``some condition"})$ is an indicator function, taking the value $1$ when ``some condition" is satisfied, and $0$ otherwise.
  \item $\|\cdot\|_\infty$ denotes the sup-norm.
\end{itemize}

\section{Overall RVI-SAC algorithm and implementation}
\label{sec:Overall_RVI_SAC_algorithm_and_implementation}

\begin{algorithm}[htp]
  \caption{RVI-SAC}
  \label{alg:rvi_sac}
  \begin{algorithmic}[1]
    \small
    \STATE \textbf{Initialize:}{Q-Network parameters $\phi_1, \phi_2, \phi_1', \phi_2'$,
      Delayed $f(Q)$ update parameter $\xi$,
      Policy-Network parameters $\theta$,
      Temperature parameter $\alpha$,
      Q-Network parameters for reset $\phi_{\text{reset}}, \phi_{\text{reset}}'$,
      Delayed $f(Q)$ update parameter $\xi_{\text{reset}}$ for reset,
      and Reset Cost $r_{\text{cost}}$.}

    \FOR{each iteration}
    \STATE Sample action $a \sim \pi(\cdot|s)$
    \STATE Sample next state $s' \sim p(\cdot|s,a)$ and reward $r$
    \IF{$s' \notin \mathcal{S}_{\text{term}}$}
    \STATE Store transition $(s, a, r, s', \text{is\_reset\_step} = \text{false})$ in replay buffer $\mathcal{D}$
    \ELSIF{$s' \in \mathcal{S}_{\text{term}}$}
    \STATE Reset environment to initial state $s_0$
    \STATE Store transition $(s, a, r, s_0, \text{is\_reset\_step} = \text{true})$ in replay buffer $\mathcal{D}$
    \ENDIF
    \STATE Sample a mini-batch $\mathcal{B}$ from $\mathcal{D}$
    \STATE Update $\phi_1, \phi_2$ by minimizing Q-Network loss $J(\phi_i)$ (Eq. \ref{eq:Q_Network_Loss})
    \STATE Update $\xi$ using delayed f(Q) update method (Eq. \ref{eq:Xi_update})
    \STATE Update $\phi_{\text{reset}}$ by minimizing Reset Q-Network loss $J(\phi_{\text{reset}})$ (Eq. \ref{eq:Reset_Q_Network_Loss})
    \STATE Update $\xi_{\text{reset}}$ using delayed f(Q) update method (Eq. \ref{eq:Reset_Xi_update})
    \STATE Update $\theta$ by minimizing Policy-Network loss $J(\theta)$ (Eq. \ref{eq:Policy_Network_Loss})
    \STATE Update $\alpha$ by minimizing Temperature Parameter loss $J(\alpha)$ (Eq. \ref{eq:Temperature_Parameter_Loss})
    \STATE Update $r_{\text{cost}}$ by minimizing Reset Cost loss $J(r_{\text{cost}})$ (Eq. \ref{eq:Reset_Cost_Loss})
    \STATE Update $\phi_1', \phi_2', \phi_{\text{reset}}'$(Eq. \ref{eq:Target_Network_Update})
    \ENDFOR
  \end{algorithmic}
\end{algorithm}

% 本節では，\ref{sec:RVI_Q_learning_based_Q_Network_update}，\ref{sec:Average_Reward_Soft_Policy_Improvement_Theorem}，\ref{sec:Automatic_Reset_Cost_Adjustment}節をもとに，RVI-SACの全体像を示す．

In this section, based on Sections \ref{sec:RVI_Q_learning_based_Q_Network_update}, \ref{sec:Average_Reward_Soft_Policy_Improvement_Theorem}, and \ref{sec:Automatic_Reset_Cost_Adjustment}, we present the overall algorithm of RVI-SAC.

The main parameters to be updated in this algorithm are:
\begin{itemize}
  \item The parameters $\phi_1, \phi_2$ of the Q-Network and their corresponding target network parameters $\phi_1', \phi_2'$,
  \item The parameter $\xi$ for the Delayed f(Q) Update,
  \item The parameters $\theta$ of the Policy-Network,
  \item Additionally, this method introduces automatic adjustment of the temperature parameter $\alpha$, as introduced in SAC-v2\cite{Haarnoja2018SACv2}.
\end{itemize}
Note that the update of the Q-Network uses the Double Q-Value function approximator \cite{Fujimoto2018TD3}.

In cases where environment resets are needed, the following parameters are also updated:
\begin{itemize}
  \item The Reset Cost $r_{\text{cost}}$,
  \item The parameters $\phi_{\text{reset}}$ of the Q-Network for estimating the frequency of environment resets and their corresponding target network parameters $\phi'_{\text{reset}}$,
  \item The parameter $\xi_{\text{reset}}$ for the Delayed f(Q) Update.
\end{itemize}
The Q-Network used to estimate the frequency of environment resets is not directly used in policy updates, therefore the Double Q-Value function approximator is not employed for it.
% このアルゴリズムで更新すべき主なパラメータは，
% Q-Networkのパラメータ$\phi_1, \phi_2$と，そのターゲットネットワークのパラメータ$\phi_1', \phi_2'$，
% Delayed f(Q) Updateのパラメータ$\xi$，
% Policy-Networkのパラメータ$\theta$
% それに加えて，本手法ではSAC-v2で導入された温度パラメータの自動調整を導入するため，温度パラメータ$\alpha$も更新する．
% ここで，Q-Networkの更新にはDouble Q-Value function approximator\cite{Fujimoto2018TD3}を用いることに注意する．
% Resetが必要な場合は，
% Reset Cost $r_{\text{cost}}$，
% Resetが発生する頻度を推定するためのQ-Networkのパラメータ$\phi_{\text{reset}}$と，そのターゲットネットワークのパラメータ$\phi'_{\text{reset}}$，
% Delayed f(Q) Updateのパラメータ$\xi_{\text{reset}}$も更新する．
% Resetが発生する頻度を推定するためのQ-Networkは，方策の更新に直接用いられないため，Double Q-Value function approximatorを用いない．

% \ref{sec:RVI_Q_learning_based_Q_Network_update}，\ref{sec:Average_Reward_Soft_Policy_Improvement_Theorem}節より，Q-Networkのパラメータ$\phi_1, \phi_2$の更新は，以下の損失関数を最小化するように更新する．

From Sections \ref{sec:RVI_Q_learning_based_Q_Network_update} and \ref{sec:Average_Reward_Soft_Policy_Improvement_Theorem}, the parameters $\phi_1, \phi_2$ of the Q-Network are updated by minimizing the following loss function:
\begin{equation}\label{eq:Q_Network_Loss}
  J(\phi_i) =  \frac{1}{|\mathcal{B}|} \sum_{(s, a, r, s', \text{is\_reset\_step}) \in \mathcal{B}}\left(Y(s, a, r, s',\text{is\_reset\_step}) - Q_{\phi_i}(s, a)\right)^2, \ i=1,2,
\end{equation}
where
$$
  \begin{aligned}
     & Y(s, a, r, s', \text{is\_reset\_step}) = \hat{r} - \xi + \min_{j=1,2} Q_{\phi'_j}(s', a') - \alpha \log \pi_{\theta}(a'|s'), \\
     & \hat{r} = r - r_{\text{cost}} \mathbbm{1}\left( \text{is\_reset\_step} \right), a' \sim \pi_{\theta}(\cdot|s').
  \end{aligned}
$$
Note that $\hat{r}$ is the reward penalized by the Reset Cost.
$\xi$ is updated as follows using the parameter $\kappa$, based on the Delayed f(Q) update:
\begin{equation}\label{eq:Xi_update}
  \begin{aligned}
     & \xi \leftarrow \xi + \kappa \left(f(Q_{\phi'}^{\text{ent}}; \mathcal{B}) - \xi\right), \\
     & Q_{\phi'}^{\text{ent}}(s,a) := Q_{\phi'}(s,a) - \alpha \log \pi_{\theta}(a|s).
  \end{aligned}
\end{equation}
% ここで，$Q_{\phi'}^{\text{ent}}(s,a)$は式\ref{eq:soft_average_reward_Q_ent}のentropy augumented Q関数である．
% また，関数$Q:\mathcal{S}\times\mathcal{A} \rightarrow \mathbb{R}$に対して$f(Q; \mathcal{B})$は以下のように計算される．
$Q_{\phi'}^{\text{ent}}(s,a)$ represents the entropy augmented Q function as in Equation \ref{eq:soft_average_reward_Q_ent}.
For a function $Q:\mathcal{S}\times\mathcal{A} \rightarrow \mathbb{R}$, $f(Q; \mathcal{B})$ is calculated as follows:
\begin{equation}\label{eq:F_Sampled_algorithm}
  \begin{aligned}
     & f(Q; \mathcal{B}) = \frac{1}{|\mathcal{B}|} \sum_{(s, a, r, s', \text{is\_reset\_step}) \in \mathcal{B}} Q(s', a'), \\
     & a' \sim \pi_{\theta}(\cdot|s').
  \end{aligned}
\end{equation}

% Policy-Networkのパラメータ$\theta$の更新は，\ref{sec:Average_Reward_Soft_Policy_Improvement_Theorem}節で述べたようにSACと同じ手法，すなわち以下の損失関数を最小化するように更新する．
The parameters $\theta$ of the Policy-Network is updated to minimize the following loss function, as described in Section \ref{sec:Average_Reward_Soft_Policy_Improvement_Theorem}, using the same method as in SAC:
\begin{equation}\label{eq:Policy_Network_Loss}
  J(\theta) = \frac{1}{|\mathcal{B}|} \sum_{(s, a, r, s', \text{is\_reset\_step}) \in \mathcal{B}} \left(\alpha \log \pi_{\theta}(a'|s) - \min_{j=1,2} Q_{\phi_j}(s, a')\right), a' \sim \pi_{\theta}(\cdot|s).
\end{equation}
Furthermore, since the theory and update method for the temperature parameter $\alpha$ do not depend on the discount rate $\gamma$, it is updated in the same way as in SAC:
\begin{equation}\label{eq:Temperature_Parameter_Loss}
  J(\alpha) = \frac{1}{|\mathcal{B}|} \sum_{(s, a, r, s', \text{is\_reset\_step}) \in \mathcal{B}} \alpha  \left(- \log \pi_{\theta}(a'|s) - \overline{\mathcal{H}} \right),  a' \sim \pi_{\theta}(\cdot|s).
\end{equation}
where $\overline{\mathcal{H}}$ is the entropy target.

% Resetが発生する頻度を推定するためのQ-Networkのパラメータ$\phi_{\text{reset}}$の更新は，以下の損失関数を最小化するように更新し，
The parameters $\phi_{\text{reset}}$ of the Q-Network for estimating the frequency of environment resets are updated to minimize the following loss function,
\begin{equation}\label{eq:Reset_Q_Network_Loss}
  \begin{aligned}
     & J(\phi_{\text{reset}}) = \frac{1}{|\mathcal{B}|} \sum_{(s, a, r, s', \text{is\_reset\_step}) \in \mathcal{B}}\left(Y_{\text{reset}}(s, a, r, s', \text{is\_reset\_step}) - Q_{\phi_{\text{reset}}}(s, a)\right)^2, \\
     & Y_{\text{reset}}(s, a, r, s',\text{is\_reset\_step}) = \mathbbm{1}\left( \text{is\_reset\_step} \right) - \xi_{\text{reset}} +  Q_{\phi'_{\text{reset}}}(s',a'),                           \\
     & a' \sim \pi_{\theta}(\cdot|s'),
  \end{aligned}
\end{equation}
and $\xi_{\text{reset}}$ is updated as
% $\xi_{\text{reset}}$は，
\begin{equation}\label{eq:Reset_Xi_update}
  \begin{aligned}
    \xi_{\text{reset}} \leftarrow \xi_{\text{reset}} + \kappa \left(f(Q_{\phi'_{\text{reset}}}; \mathcal{B}) - \xi_{\text{reset}}\right)
  \end{aligned}
\end{equation}
using the calculation for $f(Q_{\phi'_{\text{reset}}}; \mathcal{B})$ provided in Equation \ref{eq:F_Sampled_algorithm}.
% のように更新する．ここで$f(Q_{\phi'_{\text{reset}}}; \mathcal{B})$は式\ref{eq:F_Sampled_algorithm}を用いて計算される．

% $\xi_{\text{reset}}$を$\rho_{\text{reset}}^\pi $の推定値として用いると，Reset Cost $r_{\text{cost}}$は\ref{sec:Automatic_Reset_Cost_Adjustment}節より，以下の損失関数を最小化するように更新する．

When using $\xi_{\text{reset}}$ as an estimator for $\rho_{\text{reset}}^\pi$, the Reset Cost $r_{\text{cost}}$ is updated to minimize the following loss function, as described in Section \ref{sec:Automatic_Reset_Cost_Adjustment}:
\begin{equation}\label{eq:Reset_Cost_Loss}
  J(r_{\text{cost}}) =  - r_{\text{cost}} \left(\xi_{\text{reset}} - \epsilon_{\text{reset}} \right).
\end{equation}

The parameters of the target network \( \phi_1', \phi_2', \phi'_{\text{reset}} \) are updated according to the parameter \( \tau \) as follows:
\begin{equation}\label{eq:Target_Network_Update}
  \begin{aligned}
     & \phi_1' \leftarrow \tau \phi_1 + (1 - \tau) \phi_1'                                        \\
     & \phi_2' \leftarrow \tau \phi_2 + (1 - \tau) \phi_2'                                        \\
     & \phi_{\text{reset}}' \leftarrow \tau \phi_{\text{reset}} + (1 - \tau) \phi_{\text{reset}}'
  \end{aligned}
\end{equation}

% アルゴリズム全体の疑似コードを\ref{alg:rvi_sac}に示す．
The pseudocode for the entire algorithm is presented in Algorithm \ref{alg:rvi_sac}.

\section{Convergence Proof of RVI Q-learning with Delayed $f(Q)$ Update}
\label{sec:convergence_proof_of_RVI_Q_learning_with_Delayed_fQ_Update}

% このセクションでは，式\ref{eq:RVI_Q_learning_with_Delayed_fQ_Update}に示される，Delayed f(Q) Updateの収束性について議論する．
% このアルゴリズはtwo time scaleのSAとなっており，平均報酬型のQ-learningにおける$Q$とオフセットを異なるtime scaleで更新する．
% このアイディアは\cite{GOSAVI2004654}で提案されたアルゴリズムと良く似通っており，以下の収束の議論では\cite{Konda1999ActorCriticSA,GOSAVI2004654}での議論と同じものを多く用いている．

In this section, we present the asymptotic convergence of the Delayed f(Q) Update algorithm, as outlined in Equation \ref{eq:RVI_Q_learning_with_Delayed_fQ_Update}.
This algorithm is a two-time-scale stochastic approximation (SA) and updates the Q-values and offsets in average reward-based Q-learning at different time scales.
This approach is similar to the algorithm proposed in \citet{GOSAVI2004654}.
Moreover, the convergence discussion in this section largely draws upon the discussions in \citet{Konda1999ActorCriticSA,GOSAVI2004654}.

\subsection{Proposed algorithm}

% 本節で我々は収束を証明すべきアルゴリズムを再定式化する．

% 有限な状態行動空間を持ちAssumption\ref{assumption:ergodic}を満たすMDPを考えましょう．
% 全ての$(s,a) \in \mathcal{S} \times \mathcal{A}$に対して，スカラー$\{\xi_k\}$とテーブル型で表されるQ関数$\{Q_k\}$に関する更新式を考える．
In this section, we reformulate the algorithm for which we aim to prove convergence.

Consider an MDP with a finite state-action space that satisfies Assumption \ref{assumption:ergodic}.
For all $(s,a) \in \mathcal{S} \times \mathcal{A}$, let us define the update equations for a scalar sequence ${\xi_k}$ and a tabular Q function ${Q_k}$ as follows:
\begin{eqnarray}
  \xi_{k+1}&=&\xi_{k}+a(k)\left(f(Q_k; X_k)-\xi_{k}\right), \label{eq:Proposed_RVI_fQ_Update} \\
  Q_{k+1}(s,a)&=&Q_{k}(s,a)+b(\nu(k, s, a))\left(r(s,a)- g_{\eta}(\xi_k) +\max_{a'}Q_{k}(s',a') - Q_{k}(s,a)\right) \mathbbm{1}\left((s, a) = \phi_k \right). \label{eq:Proposed_RVI_Q_Update}
\end{eqnarray}
% $g_{\eta}(\cdot)$は，アルゴリズムの収束を保証するために式\ref{eq:RVI_Q_learning_with_Delayed_fQ_Update}から新たに追加されたclip functionであり，任意の$\eta>0$に対して，
$g_{\eta}(\cdot)$ is a clip function newly added from Equation \ref{eq:RVI_Q_learning_with_Delayed_fQ_Update} to ensure the convergence of this algorithm.
For any $\eta > 0$, it is defined as:
\begin{equation}
  \label{eq:definition_g_eps}
  g_{\eta}(x)=
  \left\{
  \begin{array}{lll}
    \|r\|_\infty + \eta   & \text { for } & x \geq \|r\|_\infty + \eta,                   \\
    x                     & \text { for } & - \|r\|_\infty - \eta<x< \|r\|_\infty + \eta, \\
    - \|r\|_\infty - \eta & \text { for } & x \leq - \|r\|_\infty - \eta.
  \end{array}
  \right.
\end{equation}
% となる．
% $k$は離散的なタイムステップ，
% $\phi_k$は時刻$k$において更新される状態/行動ペアの系列，$s'$は対象とする環境で状態$s$行動$a$を選択したときのサンプルされる次の状態を表す．
% $\nu(k,s,a)$は，ある$(s,a)$に対して時刻$k$までに$Q(s,a)$を更新した回数を表しており，$\nu(k,s,a)=\sum_{m = 0}^{k} \mathbbm{1}\left((s, a) = \phi_m \right)$となる．
% ここで，確率変数$X_k$と$f(\cdot;\cdot)$と増大する$\sigma -$加法族$\mathcal{F}_k = \sigma(\xi_n, Q_n, n \leq k, w_{n,1}, w_{n,2}, n < k)$に対して，以下を仮定する($w_{n,1}, w_{n,2}$は後の式\ref{eq:definition_w}に定義されている)．
At discrete time steps $k$, $\phi_k$ denotes the sequence of state/action pairs updated at time $k$, and $s'$ represents the next state sampled when the agent selects action $a$ in state $s$.
The function $\nu(k,s,a)$ counts the number of updates to $Q(s,a)$ up to time $k$, defined as $\nu(k,s,a)=\sum_{m = 0}^{k} \mathbbm{1}\left((s, a) = \phi_m \right)$.
The functions $a(\cdot)$ and $b(\cdot)$ represent the step-size.
For the random variable $X_k$ and the function $f(\cdot;\cdot)$, we introduce the following assumption within the increasing $\sigma$-field $\mathcal{F}_k = \sigma(\xi_n, Q_n, n \leq k, w_{n,1}, w_{n,2}, n < k)$, where $w_{n,1}, w_{n,2}$ are defined in Equation \ref{eq:definition_w}.
\begin{assumption}\label{assumption:f(Q;X)property}
  % $$
  %   \underE{}{f(Q_k; X_k) - f(Q_k)\mid \mathcal{F}_k} = 0
  % $$
  % が成り立ち，ある定数$K$に対して，
  % $$
  %   \underE{}{\| f(Q_k; X_k) - f(Q_k) \|^2  \mid \mathcal{F}_k} < K (1 + \|Q_k\|^2)
  % $$
  % が成り立つ．
  It holds that
  $$
    \underE{}{f(Q_k; X_k) - f(Q_k)\mid \mathcal{F}_k} = 0,
  $$
  and for some constant $K$,
  $$
    \underE{}{\| f(Q_k; X_k) - f(Q_k) \|^2  \mid \mathcal{F}_k} < K (1 + \|Q_k\|^2).
  $$
\end{assumption}
% この仮定は式\ref{eq:F_Sampled_unbiased_estimator}のように$f$を設定した場合に明らかに成り立つ．
This assumption is obviously satisfied when $f$ is set as in Equation \ref{eq:F_Sampled_unbiased_estimator}.

\subsection{The ODE framework and stochastic approximation(SA)}

% 本節では\cite{Konda1999ActorCriticSA,GOSAVI2004654}で示された，異なるtime scaleを持つ2種類のSAにの収束に関する結果を再び示す．
In this section, we revisit the convergence results for SA with two update equations on different time scales, as demonstrated in \citet{Konda1999ActorCriticSA,GOSAVI2004654}.

% $\left\{x_{k}\right\}$ と $\left\{y_{k}\right\}$ はそれぞれ $\mathbb{R}^{n}$ と $\mathbb{R}^{l}$ における数列で，$i=1, \ldots, n$と$j=1, \ldots, l$に対して，以下の更新式に従って，異なるtime scaleで生成される．
The sequences $\left\{x_{k}\right\}$ and $\left\{y_{k}\right\}$ are in $\mathbb{R}^{n}$ and $\mathbb{R}^{l}$, respectively.
For $i=1, \ldots, n$ and $j=1, \ldots, l$, they are generated according to the following update equations:
\begin{eqnarray}
  x^{i}_{k+1}&=&x^{i}_{k}+a(\nu_1(k,i))\left(h^{i}\left(x_{k}, y_{k}\right)+w^{i}_{k,1}\right) \mathbbm{1}\left(i=\phi_{k,1}\right), \label{eq:SA1} \\
  y^{j}_{k+1}&=&y^{j}_{k}+b(\nu_2(k,j))\left(f^{j}\left(x_{k}, y_{k}\right)+w^{j}_{k,2}\right) \mathbbm{1}\left(j=\phi_{k,2}\right), \label{eq:SA2}
\end{eqnarray}
where the superscripts in each vector represent vector indices, and $\left\{\phi_{k,1}\right\}$ and $\left\{\phi_{k,2}\right\}$ are stochastic processes taking values on the sets $S_{1}=\{1,2, \ldots, n\}$ and $S_{2}=\{1,2, \ldots, l\}$, respectively.
% $h(\cdot, \cdot)$ と $f(\cdot, \cdot)$ は $(x_{k}, y_{k})$ 上の任意の関数である．
The functions $h(\cdot, \cdot)$ and $f(\cdot, \cdot)$ are arbitrary functions of $(x_{k}, y_{k})$.
% $w_{k,1},w_{k,2}$ はノイズ項を表し，$\nu_1, \nu_2$は，
The terms $w_{k,1}, w_{k,2}$ represent noise components, and $\nu_1, \nu_2$ are defined as:
$$
  \begin{aligned}
    \nu_1(k,i) & =\sum_{m=0}^{k} \mathbbm{1}\left(i=\phi_{m, 1}\right), \\
    \nu_2(k,j) & =\sum_{m=0}^{k} \mathbbm{1}\left(j=\phi_{m, 2}\right).
  \end{aligned}
$$
% となる．

% 提案手法(式\ref{eq:Proposed_RVI_Q_Update},\ref{eq:Proposed_RVI_fQ_Update})の文脈では，$x_k$は$\xi_k$に，$y_k$は$Q_k$に対応する．
% つまり，$n$は1となり，$l$は状態/行動ペアの数と等しくなる．
% 式\ref{eq:SA1}の$\nu_1(k,i)$は$n=1$より$k$となり，
% 式\ref{eq:SA2}の$\nu_2(k,j)$は式\ref{eq:Proposed_RVI_Q_Update}の$\nu(k, s, a)$と対応していなる．

In the context of the proposed method (Equations \ref{eq:Proposed_RVI_fQ_Update} and \ref{eq:Proposed_RVI_Q_Update}), $x_k$ corresponds to $\xi_k$, and $y_k$ corresponds to $Q_k$.
This implies that $n$ is 1, and $l$ is equal to the number of state/action pairs.
Consequently, $\nu_1(k,i)$ corresponds $k$ due to $n=1$, and $\nu_2(k,j)$ corresponds to $\nu(k, s, a)$.

% このSAに以下を仮定する．
We assume the following assumptions for this SA:
\begin{assumption}
  \label{assumption:Lipschitz}
  The functions $h$ and $f$ are Lipschitz continuous.
\end{assumption}
\begin{assumption}
  \label{assumption:visitation_freq}
  There exist $\Delta>0$ such that
  $$\liminf _{k \rightarrow \infty} \frac{\nu_1(k,i)}{k+1} \geq \Delta,$$
  and
  $$\liminf_{k \rightarrow \infty} \frac{\nu_2(k,j)}{k+1} \geq \Delta.$$
  \textit{almost surely}, for all $i=1, 2, \ldots, n$ and $j=1,2, \ldots, l$.
  Furthermore, if, for $a(\cdot), b(\cdot)$ and $x>0$,
  $$
    \begin{aligned}
      N(k, x)  & =\min \left\{m>k: \sum_{i=k+1}^{m} \overline{a}(i) \geq x\right\}, \\
      N'(k, x) & =\min \left\{m>k: \sum_{i=k+1}^{m} \overline{b}(i) \geq x\right\},
    \end{aligned}
  $$
  where $\overline{a}(i) = a(\nu_1(i, \phi_{i,1})), \overline{b}(i) = b(\nu_2(i, \phi_{i,2}))$ , then the limits
  $$
    \begin{aligned}
       & \lim_{k \rightarrow \infty} \frac{\sum_{m=\nu_{1}(k,i')}^{\nu_{1}\left(N(k, x), i'\right)} a(m)}{\sum_{m=\nu_{1}(k,i)}^{\nu_{1}(N(k, x),i)} a(m)},              \\
       & \lim_{k \rightarrow \infty} \frac{\sum_{m=\nu_{2}(k,j')}^{\nu_{2}\left(N'(k, x), j'\right)} b(m)}{\sum_{m=\nu_{2}(k, j)}^{\nu_{2}\left(N'(n, x),j\right)} b(m)}
    \end{aligned}
  $$
  exist \textit{almost surely} for all $i,i',j,j'$ (Together, these conditions imply that the components are updated ``comparably often" in an ``evenly spread" manner.)

\end{assumption}
\begin{assumption}
  \label{assumption:step_size1}
  Let $c(k)$ be $a(k)$ or $b(k)$.
  The standard conditions for convergence that $c(k)$ must satisfy are as follows:
  \begin{itemize}
    \item $\sum_k c(k) = \infty, \sum_k c^2(k) < \infty$
    \item For $x \in(0,1)$,

          $$
            \sup _{k} c([x k]) / c(k)<\infty,
          $$

          where $[\cdots]$ stands for the integer part of ``...".
    \item For $x \in(0,1)$ and $A(k)=\sum_{i=0}^{k} c(i)$,

          $$
            A([y k]) / A(k) \rightarrow 1,
          $$

          uniformly in $y \in[x, 1]$.
  \end{itemize}
\end{assumption}
\begin{assumption}
  \label{assumption:step_size2}
  In addition to Assumption \ref{assumption:step_size1}, the following conditions must be satisfied:
  $$\lim _{k \rightarrow \infty} \sup \frac{b(k)}{a(k)}=0$$.
\end{assumption}
\begin{assumption}
  \label{assumption:martingale_difference}
  Let $\mathcal{F}_k = \sigma(x_{n}, y_{n}, n \leq k, w_{n,1}, w_{n,2}, n < k)$ be a increasing $\sigma$-field.
  For some constants ,$K_1$ and $K_2$, the following condition is satisfied:

  $$
    \begin{aligned}
       & \underE{}{w_{k,1} \mid\mathcal{F}_k}=0,                                                                               \\
       & \underE{}{\left\|w_{k,1}\right\|^{2} \mid\mathcal{F}_k} \leq K_1 (1+\left\|x_{k}\right\|^2+\left\|y_{k}\right\|^{2}),
    \end{aligned}
  $$

  and

  $$
    \begin{aligned}
       & \underE{}{w_{k,2} \mid\mathcal{F}_k}=0,                                                                               \\
       & \underE{}{\left\|w_{k,2}\right\|^{2} \mid\mathcal{F}_k} \leq K_2 (1+\left\|x_{k}\right\|^2+\left\|y_{k}\right\|^{2}).
    \end{aligned}
  $$
\end{assumption}
\begin{assumption}
  \label{assumption:boundedness}
  The iterations of $x_k$ and $y_k$ are bounded.
\end{assumption}
\begin{assumption}
  \label{assumption:h(x,y)}
  For all $y \in \mathbb{R}^{l}$, the ODE
  \begin{equation}
    \label{eq:ODE_h}
    \dot{x}_t=h(x_t, y)
  \end{equation}
  has an asymptotically stable critical point $\lambda(y)$ such that the map $\lambda$ is Lipschitz continuous.
\end{assumption}
\begin{assumption}
  \label{assumption:f(x,y)}
  The ODE
  \begin{equation}
    \label{eq:ODE_f}
    \dot{y}_t=f(\lambda(y_t), y_t)
  \end{equation}
  has a global asymptotically stable critical point $y^{*}$.
\end{assumption}
% このとき$t$は連続時間を表すことに注意する．
Note that, $t$ represents continuous time.

% Borkar's Lemma \cite{Borkar1997TwoTimeScalesSA}や，それを用いた文献\cite{GOSAVI2004654}によると，Assumption \ref{assumption:Lipschitz} to \ref{assumption:f(x,y)}が成り立つとき，$\{(x_{t}, y_{t})\}$は確率1で$\left(G\left(y^*\right), y^*\right)$ に収束する．
% \citet{Konda1999ActorCriticSA,GOSAVI2004654}より，以下の定理が成り立つ．
From \citet{Konda1999ActorCriticSA,GOSAVI2004654}, the following theorem holds:
\begin{theorem}
  \label{thm:convergence_of_SA}
  Let Assumption \ref{assumption:Lipschitz} to \ref{assumption:f(x,y)} hold.
  Then, $\{(x_{k}, y_{k})\}$ converges almost surely to $\left(\lambda\left(y^*\right), y^*\right)$.
\end{theorem}
This theorem is slightly different from the problem setting for the convergence of two-time-scale SA as described in \citet{Konda1999ActorCriticSA}.
In \citet{Konda1999ActorCriticSA}, it is assumed that a projection mapping $P$ is applied to the entire right-hand side of the update equation for $y_k$ (Equation \ref{eq:SA2}), such that $P(x)=y, \ y \in G, |x-y|=\inf _{z \in G}|z-x|$ for some closed convex set $G$.
Instead of this setting, we assume in Assumption \ref{assumption:boundedness} that $y_k$ is bounded.
With this assumption, there exists a projection mapping $P$ that, even if applied to the right-hand side of Equation \ref{eq:SA2}, would not affect the values of $y_k$.
Therefore, Theorem \ref{thm:convergence_of_SA} is essentially encompassed by the results in \citet{Konda1999ActorCriticSA}.
% この定理は，\cite{Konda1999ActorCriticSA}に示されているtwo time scales SAの収束に対する問題設定と少し異なる．
% \cite{Konda1999ActorCriticSA}では，$y_k$の更新式(式\ref{eq:SA2})の右辺全体に対して，あるclosed convex set $G$に対して，
% $P(x)=y, \ y \in G,\|x-y\|=\inf _{z \in G}\|z-x\|$となるような写像$P$を適用することを仮定している．
% 我々は，この設定の代わりに，Assumption \ref{assumption:boundedness}で$y_k$が有界であることを仮定している．
% この仮定により，何かしらの写像$P$を式\ref{eq:SA2}の右辺に適用したとしても，$y_k$の値に影響を与えない$P$が存在する．
% そのため，Theorem \ref{thm:convergence_of_SA}は，本質的に\cite{Konda1999ActorCriticSA}の結果に含まれる．

\subsection{Proof}

% 我々は，前節に示された結果を元に，それぞれの仮定を満たすことを検証することにより，提案手法の収束性を示す．

We show the convergence of the proposed method by verifying that each of the assumptions presented in the previous section is satisfied.

First, we prepare ODEs related to the update equations \ref{eq:Proposed_RVI_fQ_Update} and \ref{eq:Proposed_RVI_Q_Update}.
The mappings $H_1$ and $H_2$ are defined as follows:
$$
  \begin{aligned}
    H_1(\xi, Q)       & = f(Q),                                                                              \\
    H_2(\xi, Q)(s, a) & = r(s, a) - g_\eta(\xi) + \sum_{s'} p(s'|s,a){\max_{a'} Q(s', a')}. \label{eq:map_H}
  \end{aligned}
$$
% $H_1, H_2$を用いると，式\ref{eq:Proposed_RVI_Q_Update}, \ref{eq:Proposed_RVI_fQ_Update}は以下のように書き表される．
We rewrite Equations \ref{eq:Proposed_RVI_fQ_Update} and \ref{eq:Proposed_RVI_Q_Update} using the mappings $H_1$ and $H_2$ as follows:
\begin{eqnarray}
  \xi_{k+1}    &=& \xi_{k}+a(k)\left(H_1(\xi_k, Q_k) - \xi_k + w_{k,1}\right),                                                                \label{eq:Proposed_RVI_fQ_Update_rewritten} \\
  Q_{k+1}(s,a) &=&  Q_{k}(s,a)+b(\nu(k, s, a))\left(H_2(\xi_k, Q_k)(s,a) - Q_k(s,a) + w_{k,2}(s,a)\right) \mathbbm{1}\left((s, a) = \phi_{k} \right). \label{eq:Proposed_RVI_Q_Update_rewritten}
\end{eqnarray}
% ここで，$w_{k,1}, w_{k,2}$はそれぞれ
Here, the noise terms $w_{k,1}$ and $w_{k,2}$ are defined respectively as follows:
\begin{equation}
  \label{eq:definition_w}
  \begin{aligned}
    w_{k,1}      & = f(Q_k; X_k) - H_1(Q_k, \xi_k) = f(Q_k; X_k) - f(Q_k),                      \\
    w_{k,2}(s,a) & = r(s,a) - g_\eta(\xi_k) + \max_{a'} Q_k(s', a') - H_2(Q_k, \xi_k)(s,a).
  \end{aligned}
\end{equation}
% となる．

% $H_1, H_2$を用いて，式\ref{eq:Proposed_RVI_Q_Update}, \ref{eq:Proposed_RVI_fQ_Update}に関連するODE(式\ref{eq:ODE_h}が式\ref{eq:ODE_xi}，式\ref{eq:ODE_f}が式\ref{eq:ODE_Q}に対応)を以下のように定義する．

Using $H_1$ and $H_2$, we define the ODEs related to Equations \ref{eq:Proposed_RVI_fQ_Update} and \ref{eq:Proposed_RVI_Q_Update} (where Equation \ref{eq:ODE_xi} corresponds to Equation \ref{eq:Proposed_RVI_fQ_Update}, and Equation \ref{eq:ODE_Q} corresponds to Equation \ref{eq:Proposed_RVI_Q_Update}) as follows:
\begin{eqnarray}
  \dot{\xi}_t &=& H_1(\xi_t, Q) - \xi_t, \ \ \forall Q \label{eq:ODE_xi} \\
  \dot{Q}_t   &=& H_2(\lambda(Q_t), Q_t) - Q_t.\label{eq:ODE_Q}
\end{eqnarray}

\subsubsection{Boundness of the iteration (Assumption \ref{assumption:boundedness})}

% このセクションで我々は，Assumption \ref{assumption:boundedness}が式\ref{eq:Proposed_RVI_fQ_Update}, \ref{eq:Proposed_RVI_Q_Update}のiterationにおいても成り立つことを示す．
% そのために，\cite{GOSAVI2004654}において導入された，以下のようなMDPに対する仮定を設ける．
In this section, we show that Assumption \ref{assumption:boundedness} holds for the iterations defined in Equations \ref{eq:Proposed_RVI_fQ_Update} and \ref{eq:Proposed_RVI_Q_Update}.
To this end, we introduce an assumption for the MDP as introduced in \citet{GOSAVI2004654}
\begin{assumption}
  \label{assumption:contractive_mdp}
  There exists a state $s$ in the Markov chain such that for some integer $m$, and for all initial states and all stationary policies, $s$ is visited with a positive probability at least once within the first $m$ timesteps.
\end{assumption}
% この仮定を満たすとき，以下のような写像$T$は，
Under this assumption, the mapping $T$ defined as
\begin{equation}
  \label{eq:definition_T}
  T(Q)(s, a) = r(s, a) + \sum_{s'} p(s'|s,a){\max_{a'} Q(s', a')}
\end{equation}
% あるweighted sup-normに対して縮小写像となることが示されている．(証明は\citet{GOSAVI2004654}のAppendix A.5を参照)．
is shown to be a contraction mapping with respect to a certain weighted sup-norm. (For proof, see Appendix A.5 in \citet{GOSAVI2004654}).
% つまり，あるベクトル$\gamma$とスカラー$\delta \in (0,1), D>0$が存在し，
This means that there exists a vector $\gamma$ and a scalar $\delta \in (0,1), D>0$, such that
$$
  \| T(Q) \|_{\gamma} \leq \delta \| Q \|_{\gamma} + D
$$
% を満たす．
is satisfied.
% ここで，$\|v\|_{\gamma}$は，
Here, $|v|_{\gamma}$ is defined as a weighted sup-norm given by
$$
  \|v\|_{\gamma} = \max_{s,a \in \mathcal{S} \times \mathcal{A}} \frac{|v(s,a)|}{\gamma(s,a)}.
$$
% のように定義されるweighted sup-normである．
% ここで，$H_2$に対して，
Regarding $H_2$, it holds that
$$
  \begin{aligned}
    H_2(\xi, Q)(s,a)               & = T(Q)(s,a) - g_\eta(\xi)         \\
    \Rightarrow |H_2(\xi, Q)(s,a)| & \leq |T(Q)(s,a)| + |g_\eta(\xi)|,\   \forall s,a.
  \end{aligned}
$$
% が成り立つ．
% $g_\eta(\xi)$の定義(式\ref{eq:definition_g_eps})より，$g_\eta(\xi)$は有限であるため，ある$D_1 > 0$と任意の$\xi$に対して，以下を満たす．
From the definition of $g_\eta(\xi)$ (Equation \ref{eq:definition_g_eps}), $g_\eta(\xi)$ is bounded.
Therefore, for some $D_1 > 0$ and for any $\xi$, the following is satisfied:
$$
  \begin{aligned}
    \|H_2(\xi, Q)\|_{\gamma}             & \leq \|T(Q)\|_{\gamma} + D_1            \\
    \Rightarrow \|H_2(\xi, Q)\|_{\gamma} & \leq \delta \| Q \|_{\gamma} + D + D_1.
  \end{aligned}
$$
% 以上より，\cite{Tsitsiklis1994}における結果を用いると，$H_2$を用いた式\ref{eq:Proposed_RVI_Q_Update_rewritten}のiterationは$Q_k$を有界に保つ．
% また，$Q_k$が有界であるとき，式\ref{eq:Proposed_RVI_fQ_Update}の$f(Q_k; X_k)$も有界であるため$\xi_k$も同様に有界に保たれる．
Consequently, utilizing the results from \cite{Tsitsiklis1994}, the iteration expressed in Equation \ref{eq:Proposed_RVI_Q_Update_rewritten}, which employs $H_2$, maintains the boundedness of $Q_k$.
Additionally, when $Q_k$ is bounded, $f(Q_k; X_k)$ is also bounded, thereby ensuring that $\xi_k$ remains bounded as well.

\subsubsection{Convergence of the ODE (Assumption \ref{assumption:h(x,y)}, \ref{assumption:f(x,y)})}

% 我々は，式\ref{eq:ODE_xi}に関して，Assumption \ref{assumption:h(x,y)}を検証する．
% 式\ref{eq:ODE_xi}の$H_1(\xi, Q)$は$\xi$に依存せず，$Q$を固定した場合$H_1(\xi, Q)$は定数となる．
% よって，Assumption \ref{assumption:h(x,y)}を満たすことは明らかであり，
We verify that Equation \ref{eq:ODE_xi} satisfies Assumption \ref{assumption:h(x,y)}.
The function $H_1(\xi, Q)$ in Equation \ref{eq:ODE_xi} is independent of $\xi$, and when $Q$ is fixed, $H_1(\xi, Q)$ becomes a constant.
Therefore, it is obvious that Assumption \ref{assumption:h(x,y)} is satisfied, and we have
$$
  \lambda(Q) = f(Q).
$$
% となる．
% 以上より我々は式\ref{eq:ODE_Q}は以下のように書き換えることができる．
Consequently, we can rewrite Equation \ref{eq:ODE_Q} as follows:
\begin{equation}
  \label{eq:ODE_Q_rewritten}
  \dot{Q}_t = H_2(f(Q_t), Q_t) - Q_t = T(Q_t) - g_\eta(f(Q_t))e - Q_t.
\end{equation}

% 次に，式\ref{eq:ODE_f}に関して，Assumption \ref{assumption:f(x,y)}を検証する．
% 式\ref{eq:ODE_f}の収束性を示すために，\cite{Wan2020DifferentialQLearning}から以下の補題を導入する．
Next, we verify Assumption \ref{assumption:f(x,y)} for Equation \ref{eq:ODE_Q_rewritten}.
To demonstrate the convergence of Equation \ref{eq:ODE_Q_rewritten}, we introduce the following lemma from \citet{Wan2020DifferentialQLearning}:
\begin{lemma}
  \label{lemma:convergence_of_RVI}
  The following ODE
  \begin{equation}
    \label{eq:ODE_RVI}
    \dot{w}_t = T(w_t) - f(w_t)e - w_t
  \end{equation}
  is globally asymptotically stable and converges to $w_t \rightarrow q^*$. 
  Here, $q^*$ satisfies the optimal Bellman equation (as shown in Equation \ref{eq:average_reward_optimal_bellman_equation}) and the following conditions with respect to the function $f$:
  $$
    \begin{aligned}
      q^*(s,a) & = r(s, a) - \rho^* + \sum_{s' \in \mathcal{S}} p(s'|s, a) \max_{a'} q^*(s', a'), \\
      \rho^*   & = f(q^*).
    \end{aligned}
  $$
\end{lemma}
% ここで，式\ref{eq:ODE_Q_rewritten}に対しても，同様に以下の補題が成り立つことを示す．
For Equation \ref{eq:ODE_Q_rewritten}, we demonstrate that the following lemma holds:
\begin{lemma}
  \label{lemma:convergence_of_RVI2}
  The ODE shown in Equation \ref{eq:ODE_Q_rewritten} is globally asymptotically stable and converges to $Q_t \rightarrow q^*$.
  Here, $q^*$ is the same as the $q^*$ in Lemma \ref{lemma:convergence_of_RVI}.
\end{lemma}
\begin{proof}
  % まず，関数$g_\eta(\cdot)$の定義(式\ref{eq:definition_g_eps})より，$g_\eta(f(q^*)) = f(q^*)$となるため，$q^*$は式\ref{eq:ODE_Q_rewritten}の平衡点であることは自明である．
  First, from the definition of function $g_\eta(\cdot)$ (Equation \ref{eq:definition_g_eps}), it is obvious that $g_\eta(f(q^*)) = f(q^*)$, thus $q^*$ is an equilibrium point of the ODE shown in Equation \ref{eq:ODE_Q_rewritten}.

  % 次に，式\ref{eq:ODE_Q_rewritten}に示すODEがLyapunov stabilityを満たすことを示す．
  Next, we show that the ODE presented in Equation \ref{eq:ODE_Q_rewritten} satisfies Lyapunov stability.
  % つまり、任意の$\epsilon>0$に対してある$\delta>0$が存在し，$\|q^*-Q_{0}\|_{\infty} \leq \delta$が成り立つとき，全ての$t\geq0$に対して$\left|q^*-Q_{t}\right|_{\infty} \leq \epsilon$を意味するような$\delta>0$となることが必要である．
  That is, we need to show that, for any given $\epsilon > 0$, there exists $\delta > 0$ such that if $\|q^* - Q_{0}\|_{\infty} \leq \delta$, then it implies $\|q^* - Q_{t}\|_{\infty} \leq \epsilon$ for all $t \geq 0$.
  % そのために，以下の補題を示す．
  To demonstrate this, the following lemma is presented:
  % \begin{lemma}
  %   \label{lemma:lyapunov_stability}
  %   関数$f$のLipschitz定数を$L$とする．
  %   $\|q^* - Q_0\|_\infty \leq \frac{\eta}{L(1+L)}$が成り立つときのODE(式\ref{eq:ODE_Q_rewritten})の解$Q_t$と$w_0 = Q_0$としたときのODE(式\ref{eq:ODE_RVI})の解$w_t$に対して，$Q_t = w_t$が成り立つ．
  % \end{lemma}
  \begin{lemma}
    \label{lemma:lyapunov_stability}
    Let $L$ be the Lipschitz constant of the function $f$.
    If $\|q^* - Q_0\|_\infty \leq \frac{\eta}{L(1+L)}$, then for the solution $Q_t$ of the ODE (Equation \ref{eq:ODE_Q_rewritten}) and the solution $w_t$ of the ODE (Equation \ref{eq:ODE_RVI}) with $Q_0 = w_0$, it holds that $Q_t = w_t$.
  \end{lemma}
  \begin{proof}
    % \citet{Wan2020DifferentialQLearning}より，式\ref{eq:ODE_RVI}のODEに対して以下が成り立つことがわかっている．
    From \citet{Wan2020DifferentialQLearning}, it is known that the following holds for the ODE in Equation \ref{eq:ODE_RVI}:
    $$
      \begin{aligned}
        \left|\rho^* - f(w_t)\right| & = \left|f(q^*) - f(w_t)\right|    \\
                                     & \leq L \|q^* - w_t\|_\infty       \\
                                     & \leq L(1+L) \|q^* - w_0\|_\infty. \\
      \end{aligned}
    $$
    % ここで，$\|q^* - w_0\|_\infty \leq \frac{\eta}{L(1+L)}$を満たす$w_0$を選択する．
    % このとき，$\left|\rho^* - f(w_t)\right| \leq \eta$となり，
    Here, we choose $w_0$ such that $\|q^* - w_0\|_\infty \leq \frac{\eta}{L(1+L)}$.
    Under this condition,
    $$
      \begin{aligned}
         & \left|\rho^* - f(w_t)\right| \leq \eta                                 \\
         & \Rightarrow \rho^* - \eta \leq f(w_t) \leq \rho^* + \eta               \\
         & \Rightarrow -\|r\|_\infty - \eta \leq f(w_t) \leq \|r\|_\infty + \eta. \\
      \end{aligned}
    $$
    % となる．この場合，$g_\eta(f(w_t)) = f(w_t)$となるため，$Q_0 = w_0$としたときの式\ref{eq:ODE_Q_rewritten}のODEは$Q_t = w_t$となる．
    This result implies that $g_\eta(f(w_t)) = f(w_t)$ for all $t \geq 0$.
    Therefore, for the ODE in Equation \ref{eq:ODE_Q_rewritten} with $Q_0 = w_0$, it follows that $Q_t = w_t$.
  \end{proof}
  % 式\ref{eq:ODE_RVI}のODEはLyapnov stabilityを満たす(see \citet{Wan2020DifferentialQLearning})ことから，
  % Lemma\ref{lemma:lyapunov_stability}より，任意の$\epsilon$に対して$\delta \leq \frac{\eta}{L(1+L)}$を満たす$\delta$を選択することにより，
  % 式\ref{eq:ODE_Q_rewritten}のODEもLyapnov stabilityを満たすことがわかる．
  Given that the ODE in Equation \ref{eq:ODE_RVI} satisfies Lyapunov stability (as shown in \citet{Wan2020DifferentialQLearning}),
  it follows from Lemma \ref{lemma:lyapunov_stability} that, for any $\epsilon$, by choosing $\delta \leq \frac{\eta}{L(1+L)}$,
  the ODE in Equation \ref{eq:ODE_Q_rewritten} also satisfies Lyapunov stability.

  % 最後に，global asymptotically stablilityを示すために，任意の$Q_0$に対して，$\lim_{t\rightarrow\infty}\|q^* - Q_t\|_\infty = 0$となることを示す．
  % $w_0 = Q_0$として，$v_t e = w_t - Q_t$とする．このとき，
  To demonstrate global asymptotic stability, we need to show that, for any initial $Q_0$, $\lim_{t\rightarrow\infty}\|q^* - Q_t\|_\infty = 0$.
  Setting $w_0 = Q_0$ and defining $v_t e = w_t - Q_t$.
  Then, we have
  $$
    \begin{aligned}
      \dot{v}_t e & =  \dot{w}_t - \dot{Q}_t                                                           \\
                  & = T(w_t) - f(w_t)e - w_t - \left(T(Q_t) - g_\eta(f(Q_t))e - Q_t\right)             \\
                  & = T(w_t) - T(w_t - v_t e) - \left(f(w_t)e - g_\eta(f(w_t - v_t e))e\right) - v_t e \\
                  & = T(w_t) - T(w_t) + v_t e - \left(f(w_t)e - g_\eta(f(w_t) - u v_t)e\right) - v_t e \\
                  & = - f(w_t)e + g_\eta(f(w_t) - u v_t)e.
    \end{aligned}
  $$
  From the definition of $g_\eta(\cdot)$, $\dot{v}_t$ can be expressed as follows:
  $$
    \dot{v}_t =
    \left\{
    \begin{array}{lll}
      - f(w_t) + \|r\|_\infty + \eta & \text { for } & f(w_t) - u v_t \geq \|r\|_\infty + \eta,                      \\
      -u v_t                         & \text { for } & - \|r\|_\infty - \eta < f(w_t) - u v_t < \|r\|_\infty + \eta, \\
      - f(w_t) - \|r\|_\infty - \eta & \text { for } & f(w_t) - u v_t \leq - \|r\|_\infty - \eta.
    \end{array}
    \right.
  $$
  % ここで我々は，ある$0 < \eta' < \eta$を満たす$|\rho^* - f(w_t)|$に対して，
  Here, we consider a time $T_{\eta'}$ that satisfies the following condition for some $0 < \eta' < \eta$:
  $$
    \begin{aligned}
      |\rho^* - f(w_t)| & \leq |f(q^*) - f(w_t)|                    \\
                        & \leq L\|q^* - w_t\|_\infty                \\
                        & \leq \eta', \ \ \forall t \geq T_{\eta'}.
    \end{aligned}
  $$
  % を満たす時刻を$T_{\eta'}$とする．
  % (Lemma \ref{lemma:convergence_of_RVI}より，このような$\eta'$は存在する．)
  % $t \geq T_{\eta'}$のとき，以下が成り立つ．
  (Lemma \ref{lemma:convergence_of_RVI} ensures the existence of such $T_{\eta'}$.)
  For $t \geq T_{\eta'}$, the following holds:
  $$
    \begin{aligned}
      -\eta'                                              & \leq &  & f(w_t) - \rho^* &  & \leq &  & \eta'                                  \\
      \Rightarrow  - \|r\|_\infty - \eta  < \rho^* -\eta' & \leq &  & f(w_t)          &  & \leq &  & \rho^* + \eta'  < \|r\|_\infty + \eta.
    \end{aligned}
  $$
  % よって，$v_t$に対して，以下が成り立つ．
  Therefore, for $v_t$, the following holds:
  \begin{itemize}
    \item When $f(w_t) - u v_t \geq \|r\|_\infty + \eta \Rightarrow v_t \leq \frac{f(w_t) - \|r\|_\infty - \eta}{u} < 0$, we have
          $$
            \dot{v}_t = - f(w_t) + \|r\|_\infty + \eta > 0
          $$
          % となる．
    \item When $- \|r\|_\infty - \eta < f(w_t) - u v_t < \|r\|_\infty + \eta \Rightarrow \frac{- \|r\|_\infty - \eta }{u} < v_t < \frac{\|r\|_\infty + \eta}{u}$,
          $$
            \dot{v}_t = -u v_t
          $$
          % となる．
    \item When $f(w_t) - u v_t \leq - \|r\|_\infty - \eta \Rightarrow v_t \geq \frac{f(w_t) + \|r\|_\infty + \eta}{u} > 0$, we have
          $$
            \dot{v}_t = - f(w_t) - \|r\|_\infty - \eta < 0
          $$
          % となる．
  \end{itemize}
  % リアプノフ関数を$V(v) = \frac{1}{2} v^2$と設定すると，$\dot{V}(v_t) = v_t \dot{v}_t$は$v_t = 0$のとき$\dot{V}(v_t) = 0$，$v_t \neq 0$のとき$\dot{V}(v_t) < 0$となる．
  % Lyapunov Second Methodより，$v=0$がglobal asymptotically stable pointである．
  % そのため，どのような$v_{T_{\eta'}}$に対しても，$v_t \rightarrow 0$を得ることができ，$\lim_{t\rightarrow\infty}\|q^* - Q_t\|_\infty = 0$となる．
  % 以上より，式\ref{eq:ODE_Q_rewritten}のODEは$q^*$でglobal asymptotically stableである．
  Setting the Lyapunov function as $V(v) = \frac{1}{2} v^2$, $\dot{V}(v_t) = v_t \dot{v}_t$ is such that $\dot{V}(v_t) = 0$ when $v_t = 0$, and $\dot{V}(v_t) < 0$ when $v_t \neq 0$.
  Thus, from the Lyapunov Second Method, $v=0$ is a globally asymptotically stable point.
  Therefore, for any initial $v_{T_{\eta'}}$, we can achieve $v_t \rightarrow 0$, leading to $\lim_{t\rightarrow\infty}\|q^* - Q_t\|_\infty = 0$.
  Hence, the ODE in Equation \ref{eq:ODE_Q_rewritten} is globally asymptotically stable at $q^*$.
\end{proof}

\subsubsection{Verification of other assumptions}

% 残りの仮定(Assumption \ref{assumption:Lipschitz}, \ref{assumption:visitation_freq}, \ref{assumption:step_size1}, \ref{assumption:step_size2}, \ref{assumption:martingale_difference}) を検証する．
We verify the remaining assumptions (Assumption \ref{assumption:Lipschitz}, \ref{assumption:visitation_freq}, \ref{assumption:step_size1}, \ref{assumption:step_size2} and \ref{assumption:martingale_difference}).

% まず，Assumption \ref{assumption:Lipschitz}に関して，関数$h$には$H_1(\xi, Q) - \xi$，関数$f$には$H_2(\xi, Q) - Q$が対応しており，これら関数を構成する全ての項がLipschitz連続であることは明らかである．
% よって，全体もLipschitz連続となるため，Assumption \ref{assumption:Lipschitz}は満たされる．
First, regarding Assumption \ref{assumption:Lipschitz}, the function $h$ corresponds to $H_1(\xi, Q) - \xi$, and the function $f$ corresponds to $H_2(\xi, Q) - Q$.
It is clear that all terms constituting these functions are Lipschitz continuous.
Therefore, the overall functions are also Lipschitz continuous, satisfying Assumption \ref{assumption:Lipschitz}.

% Assumption \ref{assumption:visitation_freq}は，学習過程において更新されるベクトルの要素の更新頻度に関する仮定であり，
% asynchrousなSAでは一般的に用いられている\cite{borkar2009stochastic,Abounadi2001RVI,Wan2020DifferentialQLearning}．
% $x_k$に対応する$\xi_k$は，スカラーであるため，明らかにこの仮定を満たす．
% $y_k$に対応する$Q_k$の更新に対しても，この仮定を満たすようにするため，学習過程に対して以下のような仮定を設ける．

Assumption \ref{assumption:visitation_freq} concerns the update frequency of elements in the vector being updated during the learning process, which is commonly used in asynchronous SA \cite{borkar2009stochastic,Abounadi2001RVI,Wan2020DifferentialQLearning}.
Since $\xi_k$, corresponding to $x_k$, is a scalar, it clearly satisfies this assumption.
For the updates of $Q_k$, corresponding to $y_k$, we introduce the following assumption on the learning process of our proposed method.
\begin{assumption}
  \label{assumption:visitation_freq_mdp}
  There exists $\Delta>0$ such that

  $$
    \liminf _{k \rightarrow \infty} \frac{\nu(k, s, a)}{k+1} \geq \Delta
  $$

  \textit{almost surely}, for all $s \in \mathcal{S}, a \in \mathcal{A}$.
  Furthermore, if, for $b(\cdot)$ and $x>0$,
  $$
    N'(k, x) =\min \left\{m>k: \sum_{i=k+1}^{m} \overline{b}(i) \geq x\right\},
  $$
  where $\overline{b}(i) = \nu(i, \phi_i)$ , then the limits
  $$
    \lim_{k \rightarrow \infty} \frac{\sum_{m=\nu_{2}(k,s',a')}^{\nu_{2}\left(N'(k, x), s', a'\right)} b(m)}{\sum_{m=\nu_{2}(k, s, a)}^{\nu_{2}\left(N'(n, x),s,a\right)} b(m)}
  $$

  exist \textit{almost surely} for all $s,a,s',a'$
\end{assumption}

% Assumption \ref{assumption:step_size1}もasynchrousなSAで一般的に用いられているステップサイズに対する仮定であり\cite{borkar2009stochastic,Abounadi2001RVI,Wan2020DifferentialQLearning}，
% Assumption \ref{assumption:step_size2}はtwo time-scaleなSAで一般的に用いられている仮定である\cite{borkar2009stochastic,Borkar1997TwoTimeScalesSA,GOSAVI2004654,Konda1999ActorCriticSA}．
% 我々の提案手法においても，これらの仮定を満たすステップサイズを設定する．

Assumption \ref{assumption:step_size1} is a common assumption about step sizes in asynchronous SA, typically used in various studies \cite{borkar2009stochastic,Abounadi2001RVI,Wan2020DifferentialQLearning}.
Assumption \ref{assumption:step_size2} is a standard assumption for two time-scale SA, as found in the literature \cite{borkar2009stochastic,Borkar1997TwoTimeScalesSA,GOSAVI2004654,Konda1999ActorCriticSA}.
In our proposed method, we assume the selection of step sizes that satisfy these assumptions.

% Assumption \ref{assumption:martingale_difference}はノイズに関する仮定である．
% ノイズの平均の仮定は，ノイズの定義(式\ref{eq:definition_w})から，ノイズがサンプルと条件付き平均との間の差分であるため，成立する．
% ノイズの分散の仮定は，仮定\ref{assumption:f(Q;X)property}と三角不等式を適用により，簡単に検証できる．

Assumption \ref{assumption:martingale_difference} is about the noise term.
The assumption regarding the mean of the noise is satisfied because, by the definition of noise (Equation \ref{eq:definition_w}), the noise is the difference between the sample and its conditional expectation.
The assumption regarding the variance of the noise can be easily verified by Assumption \ref{assumption:f(Q;X)property} and applying the triangle inequality.

% 以上の議論より，
% 有限な状態/行動空間のMDPに対して，
% Assumption \ref{assumption:ergodic},\ref{assumption:f(Q;X)property},\ref{assumption:contractive_mdp},\ref{assumption:step_size1},\ref{assumption:step_size2},\ref{assumption:visitation_freq_mdp}を満たすとき，
% 我々の提案した更新式\ref{eq:Proposed_RVI_Q_Update},\ref{eq:Proposed_RVI_fQ_Update}は，$Q_k \rightarrow q^*$，$\xi_k \rightarrow f(q^*)$へ確率1で収束することが示された．

Based on the above discussion, for a finite MDP, when Assumptions \ref{assumption:ergodic}, \ref{assumption:f(Q;X)property}, \ref{assumption:contractive_mdp}, \ref{assumption:step_size1}, \ref{assumption:step_size2}, and \ref{assumption:visitation_freq_mdp} are satisfied,
it has been shown that our proposed update equations \ref{eq:Proposed_RVI_Q_Update} and \ref{eq:Proposed_RVI_fQ_Update} converge \textit{almost surely} $Q_k \rightarrow q^*$ and $\xi_k \rightarrow f(q^*)$.

\section{Proof of the Average Reward Soft Policy Improvement (Theorem \ref{thm:Average_Reward_Soft_Policy_Improvement})}
\label{sec:Soft_Policy_Improvement}

In this section, we prove Theorem \ref{thm:Average_Reward_Soft_Policy_Improvement} as introduced in Section \ref{sec:Average_Reward_Soft_Policy_Improvement_Theorem}.
From the definition of the average reward soft-Q function in Equation \ref{eq:soft_average_reward_Q}, the following equation holds:
$$
  Q^{\pi}(s, a) = r(s, a) - \rho^{\pi}_{\text{soft}} + \underE{s' \sim p(\cdot|s,a) \\ a' \sim \pi(\cdot|s')}{Q^{\pi}(s', a') -  \log \pi(a'|s')}. \\
$$
% という関係が成り立つ．これを用いて，以下のような式変形を行う．
Using this, we perform the following algebraic transformation with respect to$ \rho^{\pi_{\text{old}}}_{\text{soft}}$:
$$
  \begin{aligned}
    \rho^{\pi_{\text{old}}}_{\text{soft}} & = r(s, a) - Q^{\pi_{\text{old}}}(s, a) + \underE{s' \sim p(\cdot|s,a)                                                                                                                                                                                  \\ a' \sim \pi_{\text{old}}(\cdot|s')}{Q^{\pi_{\text{old}}}(s', a') - \log \pi_{\text{old}}(a'|s')} \\
                                          & = \underE{(s,a) \sim d^{\pi_{\text{new}}}(\cdot,\cdot)}{r(s, a) - Q^{\pi_{\text{old}}}(s, a) + \underE{s' \sim p(\cdot|s,a)                                                                                                                            \\ a' \sim \pi_{\text{old}}(\cdot|s')}{Q^{\pi_{\text{old}}}(s', a') - \log \pi_{\text{old}}(a'|s')}} \\
                                          & = \rho^{{\pi_{\text{new}}}}_{\text{soft}} - \underE{(s,a) \sim d^{{\pi_{\text{new}}}}(\cdot,\cdot)}{ - \log \pi_{\text{new}}(a|s)} \\
                                          & + \underE{(s,a) \sim d^{\pi_{\text{new}}}(\cdot,\cdot)}{- Q^{\pi_{\text{old}}}(s, a) + \underE{s' \sim p(\cdot|s,a) \\ a' \sim \pi_{\text{old}}(\cdot|s')}{Q^{\pi_{\text{old}}}(s', a') - \log \pi_{\text{old}}(a'|s')}} \\
                                          & \cdots.
  \end{aligned}
$$
% ここで，\citet{Haarnoja2018SAC,Haarnoja2018SACv2}より，式\ref{eq:soft_policy_update}の更新を行うと，$\pi_{\text{\text{old}}}, \pi_{\text{new}}$に対して，以下の関係が成り立つ．
Here, according to \citet{Haarnoja2018SAC,Haarnoja2018SACv2}, with the update in Equation \ref{eq:soft_policy_update}, the following relationship holds for $\pi_{\text{old}}$ and $\pi_{\text{new}}$:
$$
  \underE{a \sim \pi_{\text{new}}(\cdot|s)}{Q^{\pi_{\text{old}}}(s, a) - \log \pi_{\text{new}}(a|s)} \geq \underE{a \sim \pi_{\text{old}}(\cdot|s)}{Q^{\pi_{\text{old}}}(s, a) - \log \pi_{\text{old}}(a|s)}, \forall s \in \mathcal{S}.
$$
Continuing the transformation with this, we get:
\small
$$
  \begin{aligned}
    %  & = \rho^{{\pi_{\text{new}}}}_{\text{soft}} - \underE{(s,a) \sim d^{{\pi_{\text{new}}}}(\cdot,\cdot)}{ - \log \pi_{\text{new}}(a|s)} + \underE{(s,a) \sim d^{\pi_{\text{new}}}(\cdot,\cdot)}{- Q^{\pi_{\text{old}}}(s, a) + \underE{s' \sim p(\cdot|s,a)    \\ a' \sim \pi_{\text{old}}(\cdot|s')}{Q^{\pi_{\text{old}}}(s', a') - \log \pi_{\text{old}}(a'|s')}} \\
    \rho^{\pi_{\text{old}}}_{\text{soft}} & = \cdots                                                                                                                                                                                                                                                                                                                                                                                                                                      \\
                                          & \leq \rho^{{\pi_{\text{new}}}}_{\text{soft}} - \underE{(s,a) \sim d^{{\pi_{\text{new}}}}(\cdot,\cdot)}{ - \log \pi_{\text{new}}(a|s)} + \underE{(s,a) \sim d^{\pi_{\text{new}}}(\cdot,\cdot)}{- Q^{\pi_{\text{old}}}(s, a) + \underE{s' \sim p(\cdot|s,a)                                                                                                                                                                                     \\ a' \sim \pi_{\text{new}}(\cdot|s')}{ Q^{\pi_{\text{old}}}(s', a') - \log \pi_{\text{new}}(a'|s') }} \\
                                          & = \rho^{{\pi_{\text{new}}}}_{\text{soft}} \cancel{ - \underE{(s,a) \sim d^{{\pi_{\text{new}}}}(\cdot,\cdot)}{ - \log \pi_{\text{new}}(a|s)}}  \cancel{+ \underE{(s,a) \sim d^{\pi_{\text{new}}}(\cdot,\cdot)}{- Q^{\pi_{\text{old}}}(s, a)}} \\
                                          & \cancel{+ \underE{(s,a) \sim d^{\pi_{\text{new}}}(\cdot,\cdot)}{- Q^{\pi_{\text{old}}}(s, a)}} + \cancel{\underE{(s,a) \sim d^{{\pi_{\text{new}}}}(\cdot,\cdot)}{ - \log \pi_{\text{new}}(a|s)}} \\
                                          & = \rho^{{\pi_{\text{new}}}}_{\text{soft}}.
  \end{aligned}
$$
\normalsize
% となり，$\rho^{{\pi_{\text{new}}}}_{\text{soft}} \leq \rho^{{\pi_{\text{new}}}}_{\text{old}}$が成り立つ．
% よって，定理\ref{thm:Average_Reward_Soft_Policy_Improvement}が示された．
Therefore, Theorem \ref{thm:Average_Reward_Soft_Policy_Improvement} has been shown.

\section{Hyperparameter settings}
\label{sec:hyperparameter_settings}

We summarize the hyperparameters used in RVI-SAC and SAC in Table \ref{tab:hyperparameters}.
We used the same hyperparameters for ARO-DDPG as \citet{Saxena2023ARODDPG}.

\begin{table}[H]
  \centering
  \begin{tabular}{lll}
    \toprule
                                                           & RVI-SAC       & SAC                 \\
    \midrule
    Discount Factor  $\gamma$                              & N/A           & [0.97, 0.99, 0.999] \\
    Optimizer                                              & Adam          & Adam                \\
    Learning Rate                                          & 3e-4          & 3e-4                \\
    Batch Size   $|\mathcal{B}|$                           & 256           & 256                 \\
    Replay Buffer Size $|\mathcal{D}|$                     & 1e6           & 1e6                 \\
    Critic Network                                         & [256, 256]    & [256, 256]          \\
    Actor Network                                          & [256, 256]    & [256, 256]          \\
    Activation Function                                    & ReLU          & ReLU                \\
    Target Smoothing Coefficient $\tau$                    & 5e-3          & 5e-3                \\
    Entrpy Target $\overline{\mathcal{H}}$                 & - dim\_action & - dim\_action       \\
    Critc Network for Reset                                & [64, 64]      & N/A                 \\
    Delayd f(Q) Update Parameter $\kappa$                  & 5e-3          & N/A                 \\
    Termination Frequency Target $\epsilon_{\text{reset}}$ & 1e-3          & N/A                 \\
    \bottomrule
  \end{tabular}
  \caption{Hyperparameters of RVI-SAC and SAC.}
  \label{tab:hyperparameters}
\end{table}

\section{Additional Results}

\subsection{Average reward evaluation}
\label{sec:appendix_average_reward_evaluation}

Figure \ref{fig:average_reward} shows the learning curves of RVI-SAC, SAC with various discount rates, and ARO-DDPG when the evaluation metric is set as average reward (total\_return / survival\_step).
Note that in the Swimmer and HalfCheetah environments, where there is no termination, the results evaluated by average reward are the same as those evaluated by total return (shown in Figure \ref{fig:main_result}).
These results, similar to those shown in Figure \ref{fig:main_result}, demonstrate that RVI-SAC also exhibits overall higher performance in terms of average reward.
\begin{figure*}[htp]
  \begin{subfigure}{1.0\textwidth}
    \centering
    \includegraphics[scale=0.5]{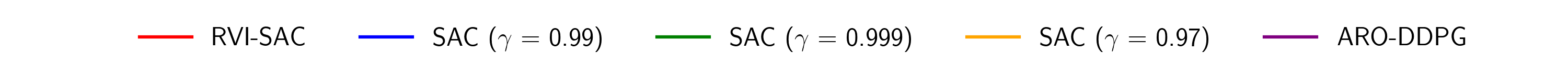}
  \end{subfigure}
  \centering
  \begin{subfigure}{.3\textwidth}
    \centering
    \includegraphics[width=\linewidth]{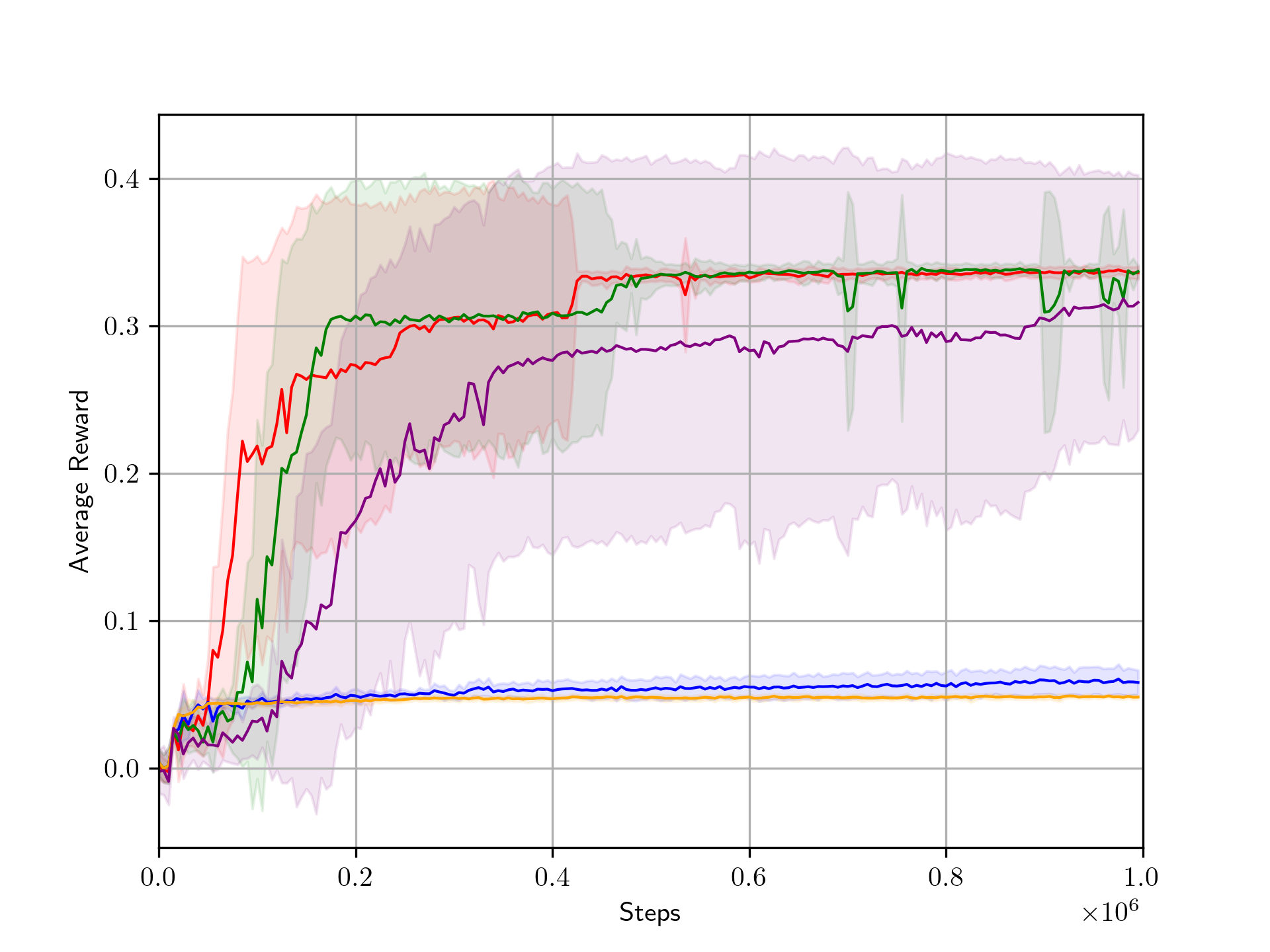}
    \caption{Swimmer}
    \label{fig:average_reward_swimmer}
  \end{subfigure}
  \begin{subfigure}{.3\textwidth}
    \centering
    \includegraphics[width=\linewidth]{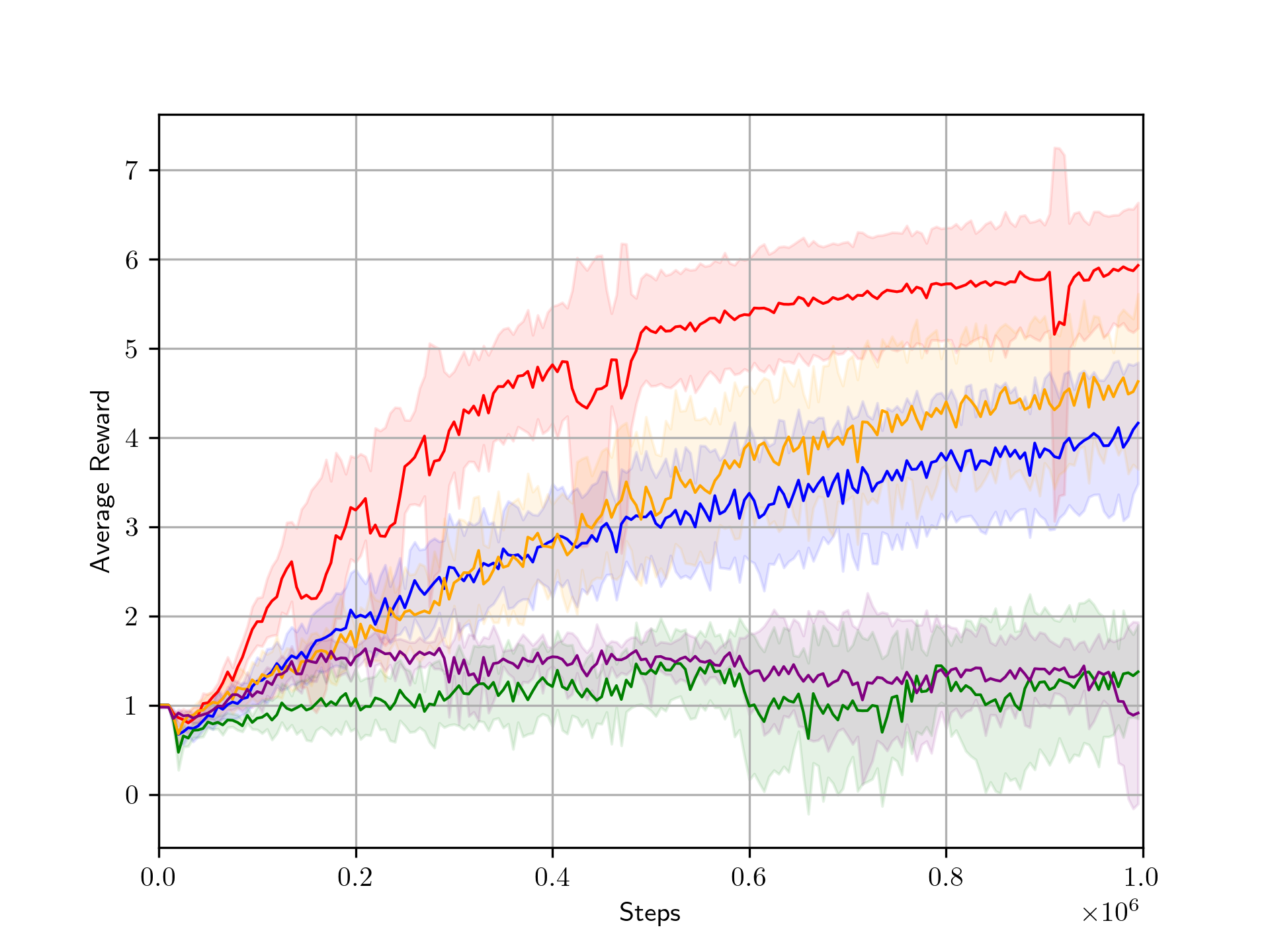}
    \caption{Ant}
    \label{fig:average_reward_ant}
  \end{subfigure}
  \begin{subfigure}{.3\textwidth}
    \centering
    \includegraphics[width=\linewidth]{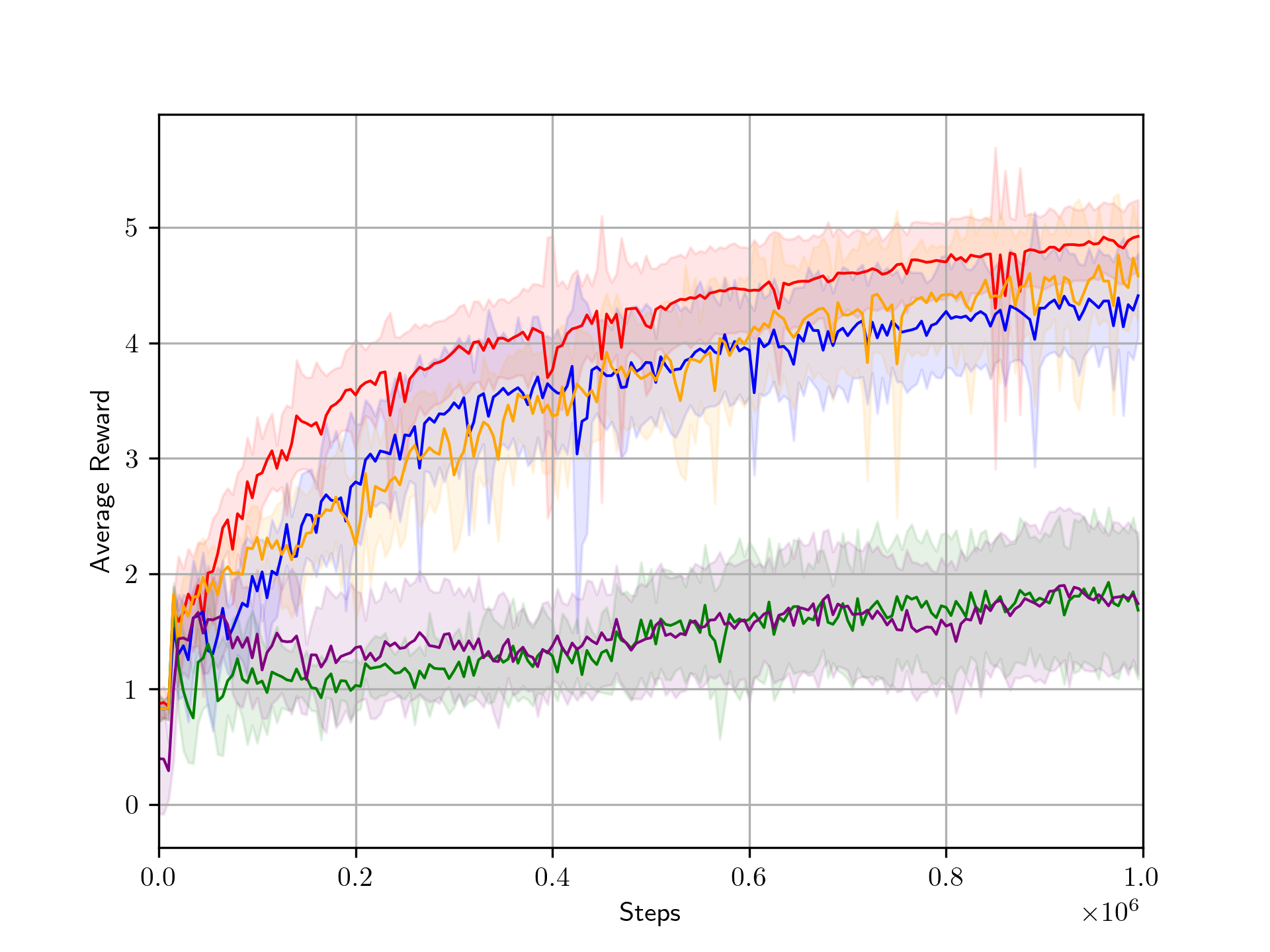}
    \caption{Walker2d}
    \label{fig:average_reward_walker2d}
  \end{subfigure}%

  \begin{subfigure}{.3\textwidth}
    \centering
    \includegraphics[width=\linewidth]{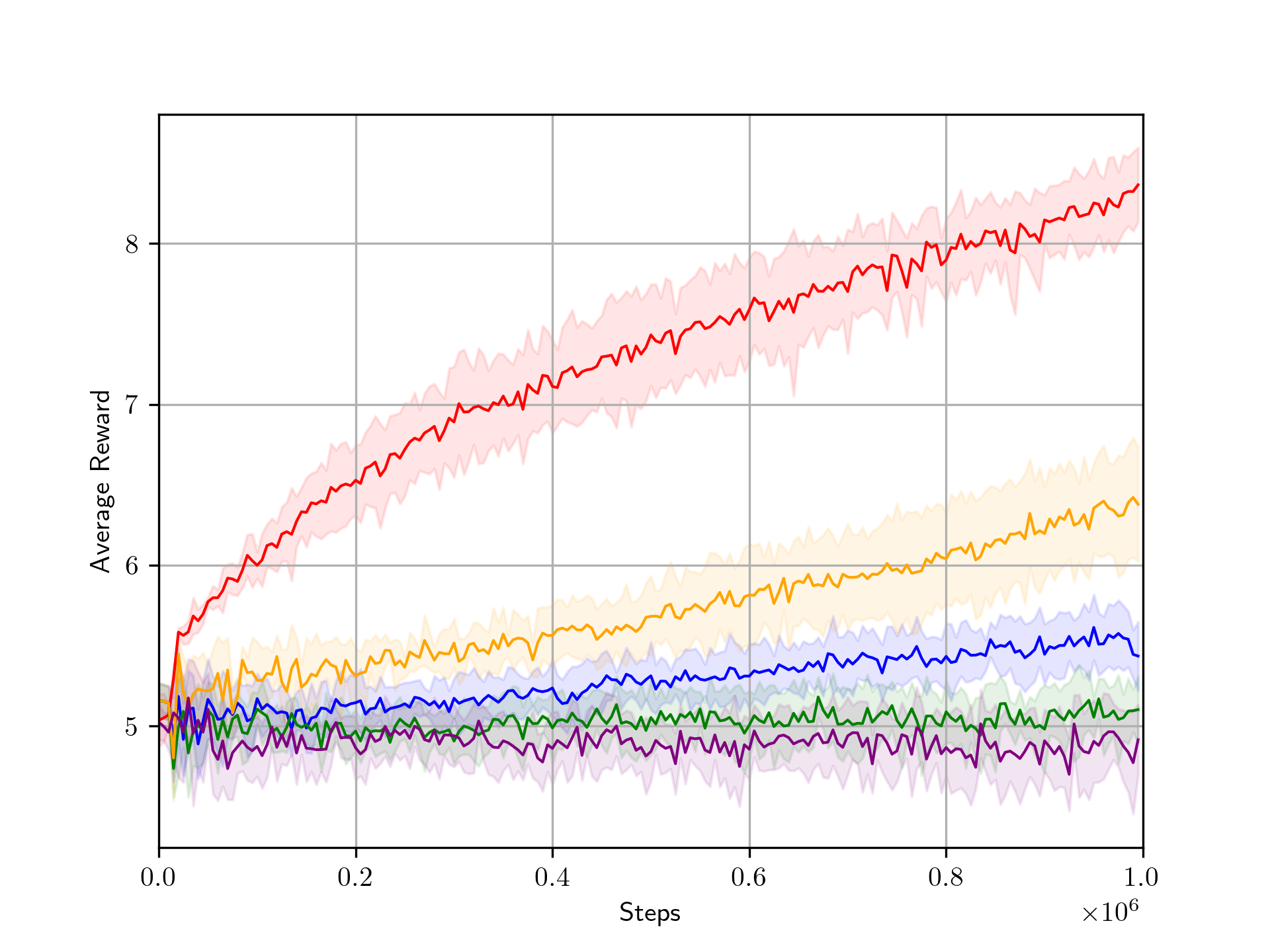}
    \caption{Humanoid}
    \label{fig:average_reward_humanoid}
  \end{subfigure}
  \begin{subfigure}{.3\textwidth}
    \centering
    \includegraphics[width=\linewidth]{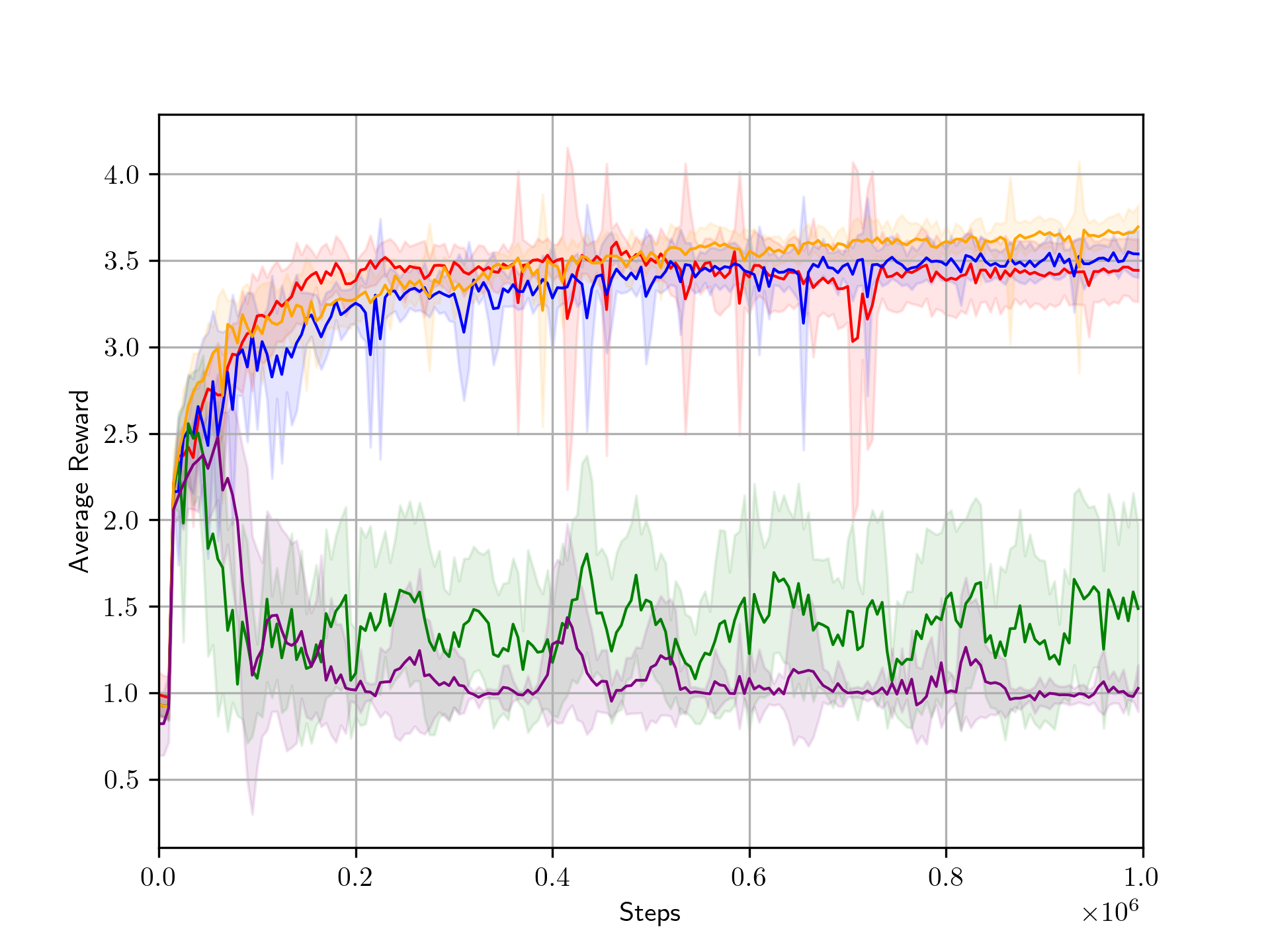}
    \caption{Hopper}
    \label{fig:average_reward_hopper}
  \end{subfigure}
  \begin{subfigure}{.3\textwidth}
    \centering
    \includegraphics[width=\linewidth]{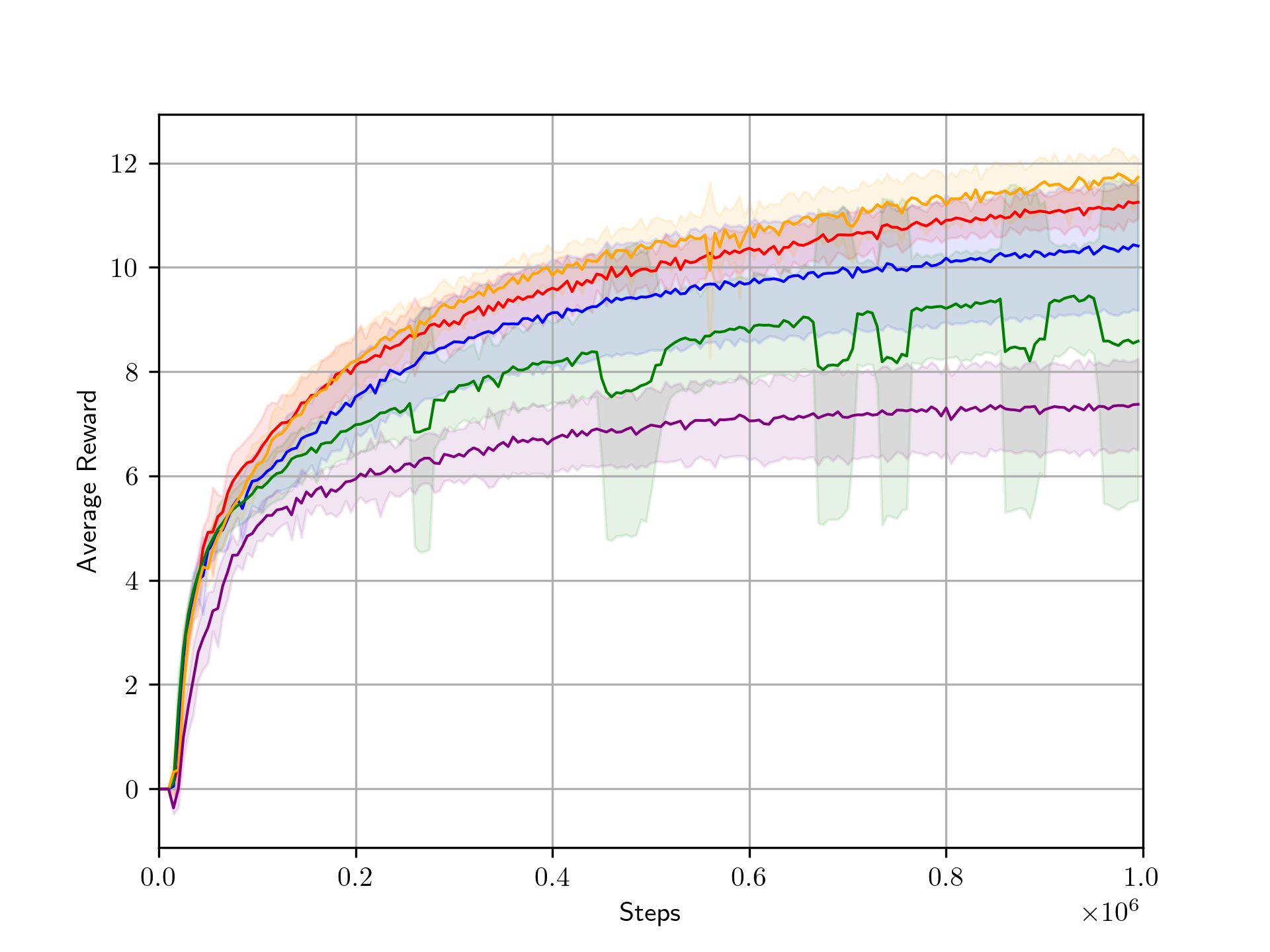}
    \caption{HalfCheetah}
    \label{fig:average_reward_halfcheetah}
  \end{subfigure}
  \caption{Learning curves for the Gymnasium's Mujoco tasks. The horizontal axis represents Steps, and the vertical axis represents the evaluation value (average reward).
    Lines and shades represent the mean and standard deviation of the evaluation values over 10 trials, respectively.}
  \label{fig:average_reward}
\end{figure*}

\subsection{Perfomance Comparison with SAC and ARO-DDPG with Reset}
\label{sec:appendix_sac_with_reset}

Figure \ref{fig:sac_with_reset_all} presents the learning curves for all environments with termination (Ant, Hopper, Walker2d and Humanoid), similar to Figure \ref{fig:sac_with_reset} in Section \ref{sec:design_evaluation}, comparing RVI-SAC with SAC  with automatic Reset Cost adjustment and ARO-DDPG with automatic Reset Cost adjustment.
Here, the discount rate of SAC is set to $\gamma = 0.99$.
The results demonstrate that RVI-SAC outperforms SAC with automatic Reset Cost adjustment and ARO-DDPG  with automatic Reset Cost adjustment across all environments.

\begin{figure*}[htp]

  \begin{subfigure}{1.0\textwidth}
    \centering
    \includegraphics[scale=0.6]{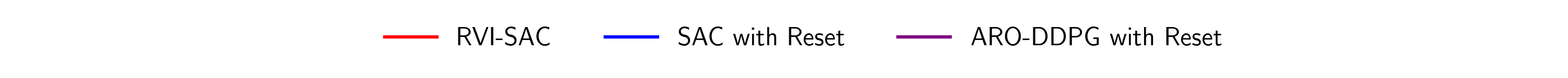}
  \end{subfigure}

  \centering
  \begin{subfigure}{.3\textwidth}
    \centering
    \includegraphics[width=\linewidth]{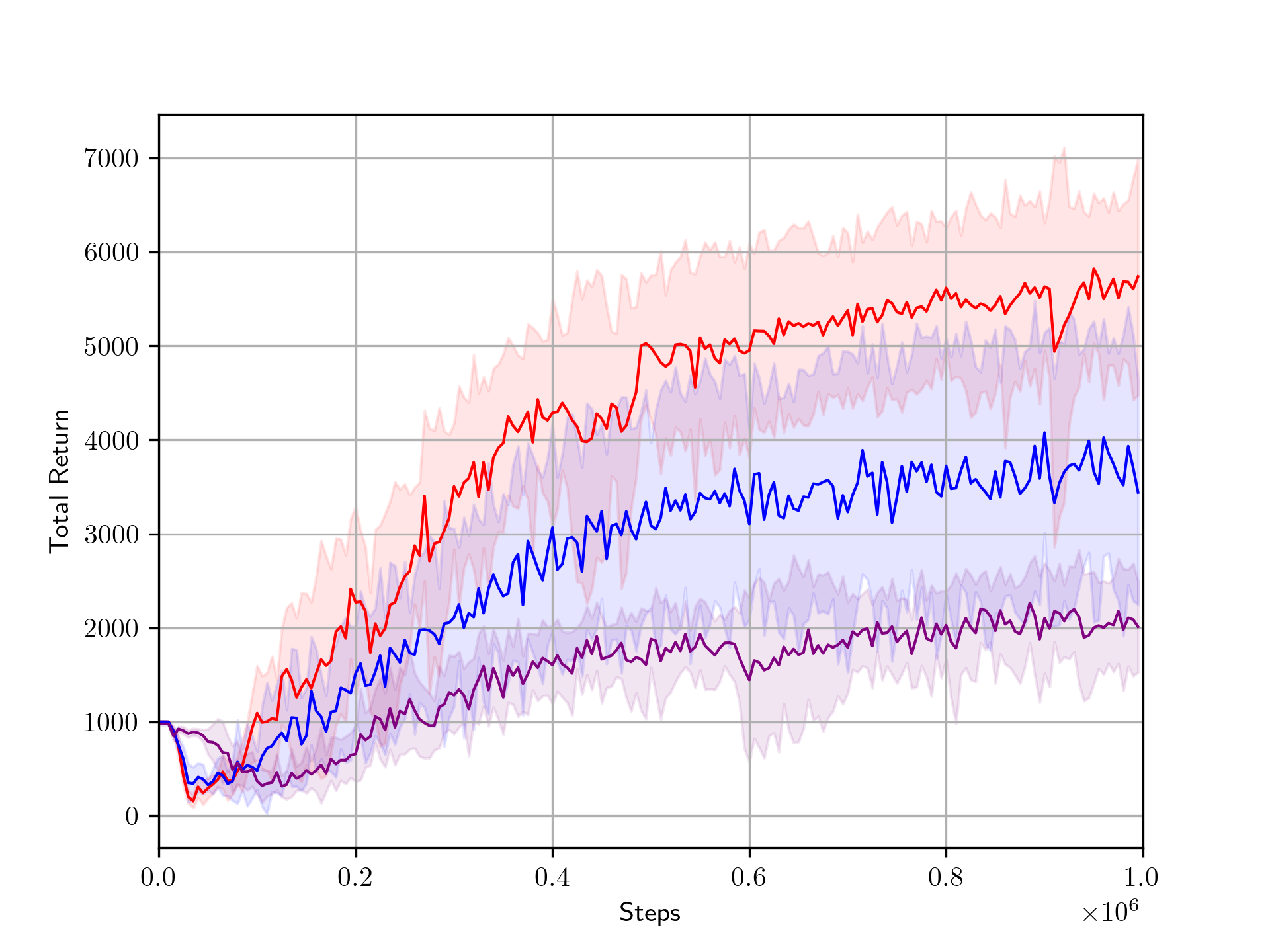}
    \caption{Ant}
    % \label{fig:sac_with_reset}
  \end{subfigure}
  \begin{subfigure}{.3\textwidth}
    \centering
    \includegraphics[width=\linewidth]{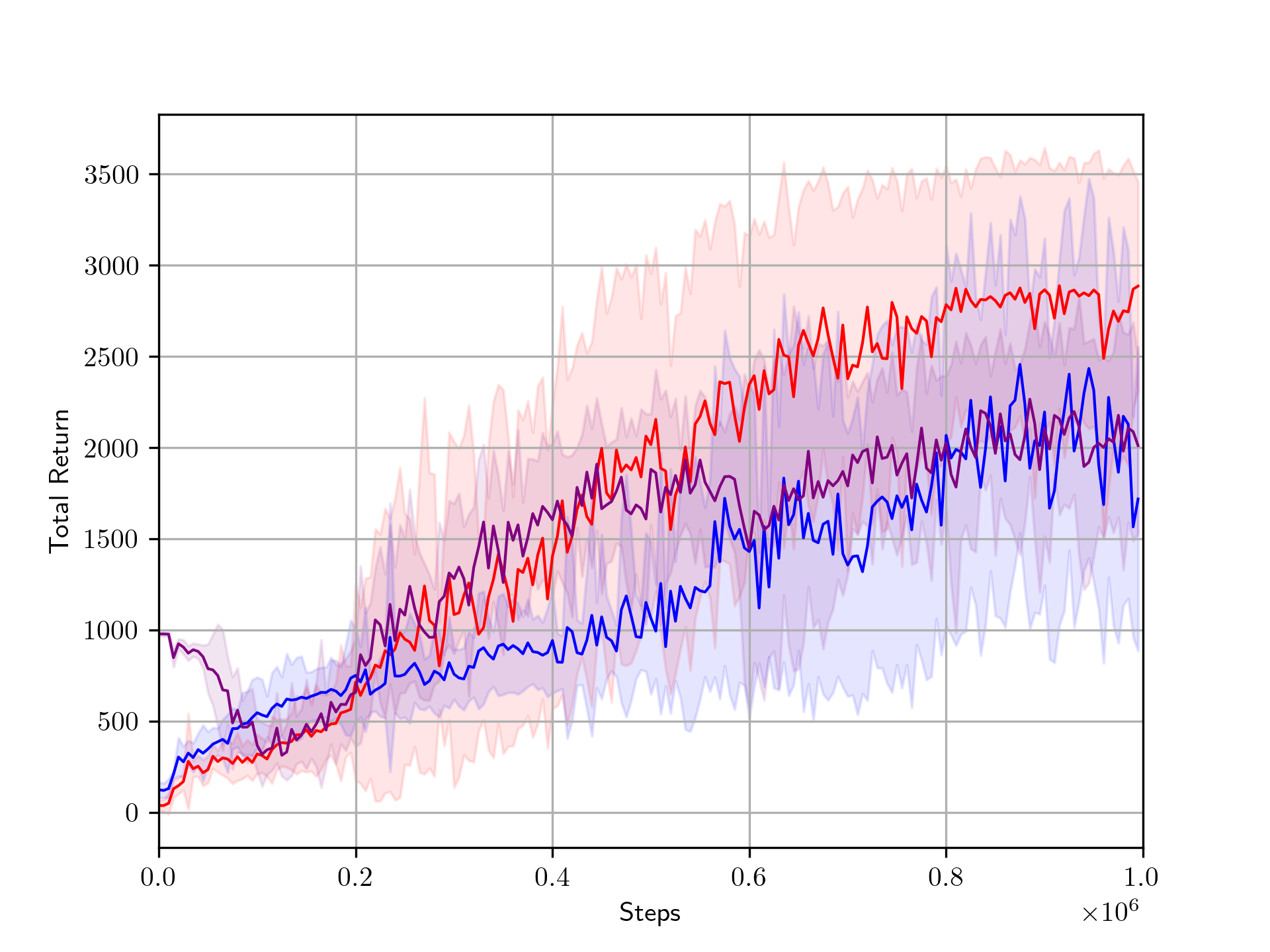}
    \caption{Hopper}
    % \label{fig:rvi_ablation}
  \end{subfigure}
  \begin{subfigure}{.3\textwidth}
    \centering
    \includegraphics[width=\linewidth]{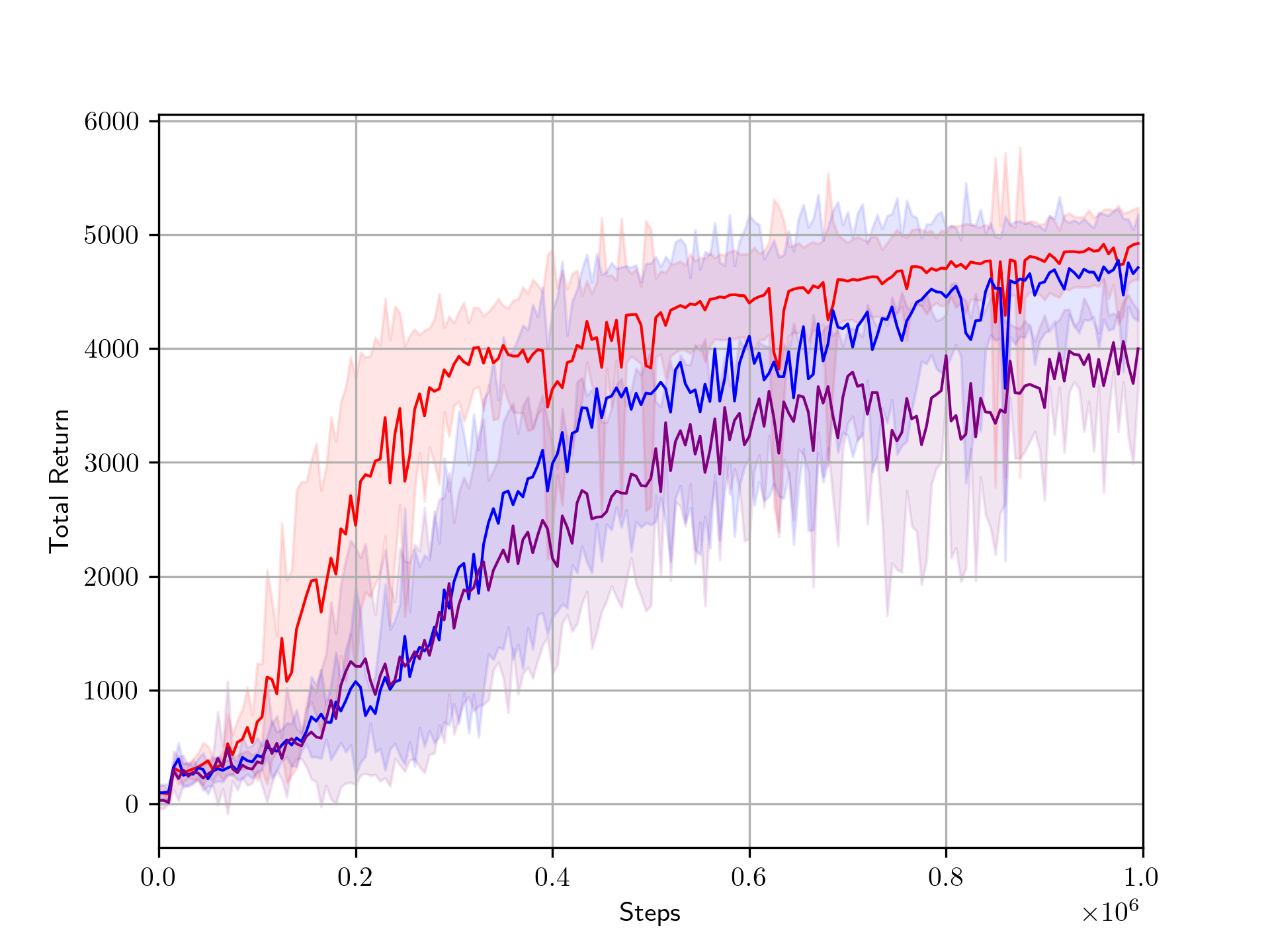}
    \caption{Walker2d}
    % \label{fig:fixed_reset}
  \end{subfigure}
  \begin{subfigure}{.3\textwidth}
    \centering
    \includegraphics[width=\linewidth]{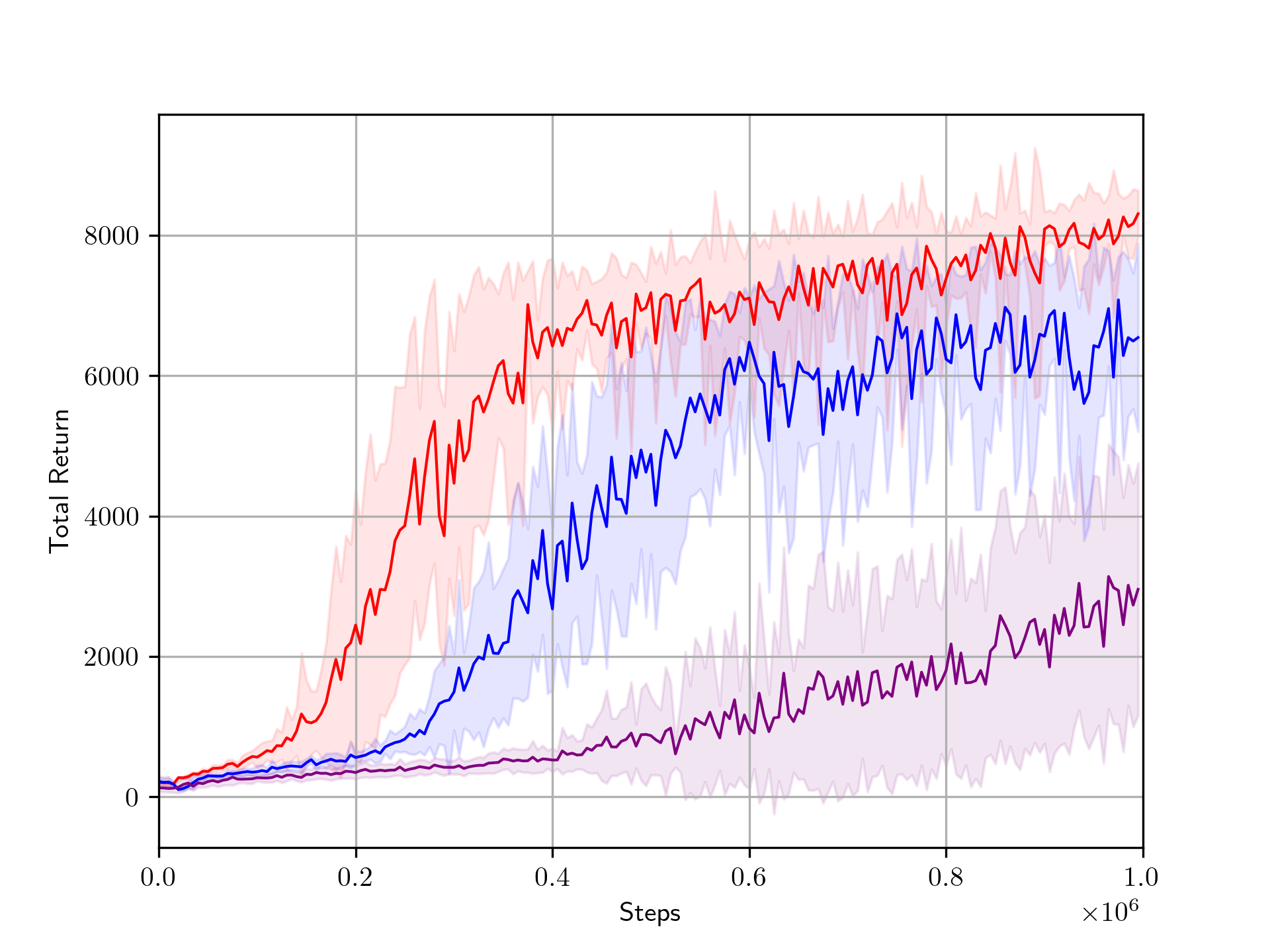}
    \caption{Humanoid}
    % \label{fig:fixed_reset}
  \end{subfigure}
  \caption{
    This figure represent learning curves for all environments with termination, compare RVI-SAC (red) with SAC (blue)  with automatic Reset Cost adjustment and ARO-DDPG(purple) with automatic Reset Cost adjustment.
  }
  \label{fig:sac_with_reset_all}
\end{figure*}

\end{document}